\documentclass[journal]{IEEEtran/IEEEtran}
\IEEEoverridecommandlockouts
\usepackage[noadjust]{cite}

\usepackage{graphics} 
\usepackage{epsfig} 
\usepackage{amsmath} 
\usepackage{amssymb}  
\usepackage{algorithm}
\usepackage{amsthm,array}
\usepackage[noend]{algpseudocode}
\usepackage{color}
\usepackage{multirow}
\usepackage{subcaption}
\usepackage{balance}

\theoremstyle{definition}
\newtheorem{definition}{Definition}
\newtheorem{assumption}{Assumption}
\newtheorem{prop}{Proposition}
\newtheorem{theorem}{Theorem}
\newtheorem{corollary}{Corollary}
\usepackage[author={Ian Abraham}]{pdfcomment}

\newcommand{\kop}{\mathfrak{K}}
\makeatletter
\def\BState{\State\hskip-\ALG@thistlm}
\makeatother

\begin{document}

\author{
Ian Abraham and Todd D. Murphey%
\thanks{Authors are with the Neuroscience and Robotics lab (NxR)
at the Department of Mechanical Engineering, Northwestern University, 2145 Sheridan Road Evanston, IL, 60208.

Videos of the experiments and sample code can be found at \url{https://sites.google.com/view/active-learning-koopman-op}  .
}%
\thanks{email: i-abr@u.northwestern.edu, t-murphey@northwestern.edu}
}

\markboth{Transactions on Robotics, In Press}%
{Abraham \MakeLowercase{\textit{et al.}}: Active Learning of Dynamics for Data-Driven Control using Koopman Operators}

\title{Active Learning of Dynamics for Data-Driven Control Using Koopman Operators}

\maketitle

\begin{abstract}
    This paper presents an active learning strategy for robotic systems that takes into account task information,
    enables fast learning, and allows control to be readily synthesized by taking advantage
    of the Koopman operator representation.
    We first motivate the use of representing nonlinear systems as linear Koopman operator systems by
    illustrating the improved model-based control performance with an actuated Van der Pol system.
    Information-theoretic methods are then applied to the Koopman operator formulation of dynamical systems
    where we derive a controller for active learning of robot dynamics.
    The active learning controller is shown to increase the rate of information about the Koopman operator.
    In addition, our active learning controller can readily incorporate policies built on the Koopman dynamics,
    enabling the benefits of fast active learning and improved control.
    Results using a quadcopter illustrate single-execution active learning and stabilization
    capabilities during free-fall.
    The results for active learning are extended for automating Koopman observables and
    we implement our method on real robotic systems.
\end{abstract}


\IEEEpeerreviewmaketitle

\section{Introduction}

    \IEEEPARstart{I}{n} order to enable active learning for robots, we need a control algorithm that readily incorporates
    task information,
    learns dynamic model representation, and is capable of incorporating policies
    for solving additional tasks during the learning process.
    In this work, we develop an active learning controller that enables a robot to learn an expressive representation of its
    dynamics using Koopman operators
    ~\cite{koopman1931hamiltonian, mezic2013analysis, mezic2015applications, budivsic2012applied}.
    Koopman operators represent a nonlinear dynamical system as a linear, infinite dimensional, system by evolving functions
    of the state (also known as function observables) in time
    ~\cite{koopman1931hamiltonian, mezic2013analysis, mezic2015applications, budivsic2012applied}.
    Often, these linear representations can capture the behavior of the dynamics globally while enabling the use of known linear
    quadratic control methods.
    As a result, the Koopman operator representation changes how we represent the dynamic constraints of the robotic systems,
    carrying more nonlinear dynamic information, and often improving control authority.

    Koopman operator dynamics are typically found through data-driven methods that generate an approximation to the theoretical
    infinite-dimensional Koopman operator~\cite{mezic2013analysis,budivsic2012applied,korda2016linear}.
    These data-driven methods require robotic systems to be actuated in order to collect data.
    The process for data collection in robotics is an active process that relies on control;
    therefore, learning the Koopman operator formulation,
    for robotics, is an active learning process.

    In this paper, we use the Koopman operator representation for improving control authority of nonlinear robotic systems.
    Moreover, we address the problem of calculating the linear representation of the Koopman operator by exploiting an
    information-theoretic active learning strategy based on the structure of Koopman operators.
    As a result, are able to demonstrate active learning through data-driven control in real-time settings where only a
    single execution of the robotic system is possible.
    Thus, the contribution of this paper is a method for active learning of Koopman operator representations of nonlinear
    dynamical systems which exploits both information-theoretic measures and improved control authority based on
    Koopman operators.

    \subsection{History and Related Work}

        Active learning in robotics has recently been a topic of
        interest~\cite{roy2001toward, baranes2013active, dima2004enabling, kober2012reinforcement, williams2017information}.
        Much work has been done in active learning for parameter identification
        \cite{armstrong1989finding, wilson2017dynamic, 6882246, yoshidaTRO2017}
        as well as active learning for state-control mappings in reinforcement learning
        \cite{kober2012reinforcement,mnih2016asynchronous,duan2016benchmarking, 7106543, Yu-RSS-17}
        and adaptive control~\cite{sin1982stochastic, ding2016recursive, ding2018least}.
        In particular, much of the mentioned work refers to exciting a robot's dynamics \textemdash using
        information theoretic measures~\cite{wilson2017dynamic, 6882246, williams2017information},
        reward functions~\cite{kober2012reinforcement,mnih2016asynchronous, 7106543, williams2017information} in
        reinforcement learning, and other methods
        \cite{ventureTRO2016, ventureTRO2016humanoid}\textemdash in order to obtain the ``best'' set of measurements that
        resolve a parameter or the ``best-case'' mapping (either of the state-control map or of the dynamics).
        This paper uses active learning to enable robots to learn Koopman operator representations of a robot's
        own dynamic process.

        Koopman operators were first proposed in 1931 in work by B.O. Koopman~\cite{koopman1931hamiltonian}.
        At the time, approximating the Koopman operator was computationally infeasible; the onset of computers enabled
        data-driven methods that approximate the Koopman operator~\cite{tu2013dynamic, budivsic2012applied, mezic2013analysis}.
        Other research involves computation of Koopman eigenfunctions and Koopman-invariant subspaces that determine the size
        of the Koopman operator~\cite{mauroy2016global,brunton2016koopman, kaiser2017data}.
        This allows for finite dimensional Koopman operators that captures nonlinear dynamics while compressing the overall
        state dimension used to represent the dynamical system.

        Recent works, on combining model-based control methods and Koopman operators have suggested that control based on Koopman
        operators is a promising avenue for many fields including robotics
        ~\cite{brunton2016koopman, mezic2015applications, korda2016linear, sootla2017pulse, rowley2015low, proctor2016dynamic,
        surana2016koopman, Abraham-RSS-17, kaiser2017data, Broad-RSS-17}.
        In particular, recent work from the authors implemented a controller using a Koopman operator representation of a robotic
        system in an
        experimental setting of a robot in sand~\cite{Abraham-RSS-17}.
        Koopman operators are closely related to latent variable (embedded) dynamic models~\cite{watter2015embed}.
        In embedded dynamic models, an autoencoder~\cite{pu2016variational, watter2015embed} is used to compress the original
        state-space into a lower-dimensional representation.
        The embedded dynamics model then only evolves the states that are useful for predicting the overall dynamical systems
        behavior.
        Koopman operators represent the state of some dynamical system in a higher- or lower- dimensional representation where
        the evolution of the embedding is a linear dynamical systems.
        Thus, Koopman operators are a special case of an embedded dynamic model where the latent variable describes the
        nonlinearities of a dynamical system and are represented as a linear differential equation.

    \subsection{Relation to Previous Work}
        We extend previous work in~\cite{Abraham-RSS-17} with new examples of control with Koopman operator representations of
        robotic systems.
        In addition, we provide an example in Section~\ref{sec:enhanced_koopman_control} which gives further intuition for the
        use of Koopman operator dynamics.
        Moreover, we address design choices when generating a Koopman operator dynamic representation of a robotic systems and
        provide a methodology towards automating these design choices.
        Last, we introduce a method for enabling the robot to actively learn Koopman operator dynamics while taking advantage of
        linear quadratic (LQ) approaches for control.
        We note that there is no overlap with the results and the theoretical content that is presented in this paper with~\cite{Abraham-RSS-17}.

    \subsection{Outline}

    The paper outline is as follows: Section~\ref{sec:kop} introduces the Koopman operator and data-driven methods to approximate
    the Koopman operator from data, including a recursively defined online approach for approximating the Koopman operator.
    Section~\ref{sec:enhanced_koopman_control} motivates using Koopman operator representations of dynamical systems for control.
    Section~\ref{sec:control_synthesis} introduces a controller that enables robots to learn the Koopman operator dynamics.
    Simulated results for active learning using our method is provided with comparisons in Section~\ref{sec:single_execution}.
    Section~\ref{sec:automating_function_discovery} discusses methods for automating the design specifications of the Koopman
    operator.
    Last, robot experiments are provided in Section~\ref{sec:robot_experiments} and concluding remarks in Section
    ~\ref{sec:conclusion} respectively.

\section{Koopman Operators} \label{sec:kop}

    This section introduces the Koopman operator and formulates the Koopman operator for control of robotic systems.

    \subsection{Infinite Dimensional Koopman Operator}
        Let us first define the continuous dynamical system whose state evolution is defined by
        \begin{align}\label{eq:cont_evo}
            x(t_i+t_s) & = F(x(t_i), u(t_i), t_s) \\
                     & = x(t_i) + \int_{t_i}^{t_i+t_s} f(x(s), u(s))ds, \nonumber
        \end{align}
        where $t_i$ is the $i^\text{th}$ sampling time and $t_s$ is the sampling interval,
        $x(t) : \mathbb{R} \to \mathbb{R}^n$ is the state of the robot at time $t$,
        $u(t) : \mathbb{R} \to \mathbb{R}^m$ is the applied actuation to the robot at time $t$,
        $f(x,u) : \mathbb{R}^n \times \mathbb{R}^m \to \mathbb{R}^n$ is the unknown dynamics of the robot,
        and $F(x(t_i), u(t_i), t_s)$ is the mapping which advances the state $x(t_i)$ to $x(t_i+t_s)$.
        In addition, let us define an observation function $g(x(t)) : \mathbb{R}^n \to \mathbb{R}^c \in \mathbb{C}$ where
        $\mathbb{C}$ is the space of all observation functions.
        The Koopman operator $\mathcal{K}$ is an infinite dimensional operator that directly acts on the elements of
        $\mathbb{C}$
        \begin{equation}
            \left[\mathcal{K}g\right](x(t_i)) = g(F(x(t_i), u(t_i), t_s)),
        \end{equation}
        where $u(t_i), t_s$ are implicitly defined in $F$ such that
        \begin{equation}
            \mathcal{K}g(x(t_i)) = g(F(x(t_i), u(t_i), t_s)) = g(x(t_{i+1})).
        \end{equation}
        In words, the Koopman operator $\mathcal{K}$ takes \emph{any} observation of state $g(x(t_i))$ at time $t_i$ and time
        shifts the observations, subject to the control $u(t_i)$, to the next observable time $t_{i+1}$.
        This formulation assumes equal time spacing $t_s = t_{i+1} - t_{i} = t_{i} - t_{i-1}$.

    \subsection{Approximating the Data-Driven Koopman Operator}
        The Koopman operator $\mathcal{K}$ is infeasible to compute in the infinite dimensional space.
        A finite subspace approximation to the operator $\kop \in \mathbb{R}^c \times \mathbb{R}^c$ acting on $\mathcal{C} \subset \mathbb{C}$ is
        used where we define a subset of function observables (or observations of state)
        $z(x) = \left[\psi_1(x), \psi_2(x), \ldots, \psi_c(x) \right]^\top \in \mathbb{R}^c \subset \mathcal{C}$.
        Each scalar valued $\psi_i \in \mathbb{C}$ and the span of all $\psi_i$ is the finite subspace
        $\mathcal{C} \subset \mathbb{C}$.
        The operator $\kop$ acting on $z(x(t_i))$ is then represented in discrete time as
        \begin{equation} \label{eq:discrete_kooop}
            z(x(t_{i+1})) = \kop z(x(t_i)) + r(x(t_i))
        \end{equation}
        where $r(x) \in \mathbb{C}$ is the residual function error.
        In principle, as $c \to \infty$, the residual error goes to zero~\cite{mezic2015applications, budivsic2012applied};
        however, it is sometimes possible to find $c < \infty$ such that $r(x) = 0$~\cite{brunton2016koopman}.
        Equation~(\ref{eq:discrete_kooop}) gives us the discrete time transition of observations of state in time.
        We overload the notation for the Koopman operator and write the differential equation for the observations of state as
        \begin{equation}
            \dot{z} = \kop z(x(t_i)) + r(x(t_i))
        \end{equation}
        where the continuous time $\kop$ is acquired by taking the matrix logarithm as $t_{i+1}-t_i \to 0$.

        Provided a data set $\mathcal{D} = \{x(t_m) \}_{m=0}^{M}$, we can compute the approximate Koopman operator $\kop$
        using least-squares minimization over the parameters of $\kop$:
        \begin{equation}\label{eq:regression}
            \min_{\kop} \frac{1}{2} \sum_{m=0}^{M-1} \Vert z(x(t_{m+1}) - \kop z(x(t_{m})) \Vert^2.
        \end{equation}
        Since (\ref{eq:regression}) is convex in $\kop$, the solution is given by
        \begin{equation}
            \kop = A G^\dagger
        \end{equation}
        where $\dagger$ denotes the Moore-Penrose pseudoinverse and
        \begin{align}\label{eq:ab_mat}
            A = \frac{1}{M} \sum_{m=0}^{M-1} z(x(t_{m+1}) z(x(t_m))^\top, \nonumber \\
            G = \frac{1}{M} \sum_{m=0}^{M-1} z(x(t_m)) z(x(t_m))^\top .
        \end{align}
        The continuous time operator is then given by $\log(\kop)/t_s$.
        Note that we can solve (\ref{eq:regression}) using gradient descent methods~\cite{rattray1998natural} or other
        optimization methods.
        We write a recursive least-squares update~\cite{ding2016recursive, lai1986extended} which adaptively
        updates $\kop$ as more data is acquired.

    \subsection{Koopman Operator for Control}
        The Koopman operator can include a predefined input $u$ that contributes to the evolution of $z(x(t))$.
        Consider the observable functions that includes the control input,
         $v(x,u) : \mathbb{R}^x \times \mathbb{R}^m \to \mathbb{R}^{c_u}$ where $c = c_x + c_u$.
        The resulting computed Koopman operator can be divided into sub-matrices
        \begin{equation}\label{eq:submat}
            \kop = \begin{bmatrix}
                \kop_x & \kop_u \\
                \cdot & \cdot
            \end{bmatrix},
        \end{equation}
        where $\kop_x \in \mathbb{R}^{c_x \times c_x}$ and $\kop_u \in \mathbb{R}^{c_x \times c_u}$.
        Note that the term ($\cdot$) in (\ref{eq:submat}) refers to terms that evolve the observations on control $z_u$
        which are ignored as there is no ambiguity in their evolution (they are determined by the controller).
        The Koopman operator dynamical system with control is then
        \begin{equation}\label{eq:koop_with_control}
            \dot{z}= f(z,u) = \kop_x z(x(t_i)) + \kop_u v(x(t_i), u(t_i)).
        \end{equation}

        Note that the data set $\mathcal{D}$ must now store $u(t_i), u(t_{i+1})$ in order to compute the Koopman operator matrix $\kop_u$.

\section{Enhancing Control Authority with Koopman Operators} \label{sec:enhanced_koopman_control}
    Koopman operators map dynamic constraints into a linear dynamical system in a modified state-space.
    The Koopman operator structure allows one to use linear quadratic (LQ) control methods to compute optimal controllers for
    nonlinear systems that can often outperform locally optimal LQ controllers obtained through linearizing the nonlinear dynamics model.

    Let us consider control of the nonlinear forced Van der Pol oscillator, the dynamics of which are defined in
    Appendix~\ref{app:vdp_dynamics}, as an example.
    We specify the control task as minimizing the following LQ objective
    \begin{multline}
        J = \int_{t_i}^{t_i +T} x(t)^\top \mathbf{Q} x(t) + u(t)^\top \mathbf{R} u(t) dt + \\x(t_i+T)^\top
    \mathbf{Q}_f x(t_i+T)
    \end{multline}
    where $\mathbf{Q} \in \mathbb{R}^{n \times n}$, $\mathbf{R} \in \mathbb{R}^{m \times m}$, and
    $\mathbf{Q}_f  \in \mathbb{R}^{n \times n}$.
    Choosing the set of function observable (Appendix~\ref{app:vdp_dynamics}), we can compute a Koopman operator $\kop$ by
    repeated simulation of the Van der Pol oscillator subject to uniformly random control inputs for $5000$ randomly
    sampled initial conditions.

    Since the Van der Pol oscillator dynamics are nonlinear, a solution to the LQ control problem is to linearize the dynamics
    about the equilibrium state $x_t = \left[ 0, 0\right]^\top$ and form a linear quadratic control regulator (LQR).
    Using the Kooman operator formulation of the Van der Pol dynamics, we can compute a controller in a similar manner using
    the following objective
    \begin{multline}
        J = \int_{t_i}^{t_i + T} z(t)^\top \tilde{\mathbf{Q}} z(t) + u(t)^\top \mathbf{R} u(t)dt + \\ z(t_i+T)^\top \tilde{\mathbf{Q}}_f z(t_i+T)
    \end{multline}
    where
    \begin{equation}\label{eq:expanded_weights}
        \tilde{\mathbf{Q}} = \begin{bmatrix}
        \mathbf{Q} & \bold{0} \\
        \bold{0} & 0
        \end{bmatrix} \in \mathbb{R}^{c_x \times c_x} \text{ and }
        \tilde{\mathbf{Q}}_f = \begin{bmatrix}
        \mathbf{Q}_f & \bold{0} \\
        \bold{0} & 0
        \end{bmatrix} \in \mathbb{R}^{c_x \times c_x}.
    \end{equation}
    Setting $\tilde{\mathbf{Q}}$ and $\tilde{\mathbf{Q}}_f$ to only include the state observables allows us to compare the same
    control objective using the linearized dynamics against the Koopman operator dynamics where the first terms in the function
    observable $z(x(t))$ is the state of the Van der Pol system itself.

    \begin{figure}[h!]
      \centering
      \begin{subfigure}{0.23\textwidth}
        \centering
        \includegraphics[width=\linewidth]{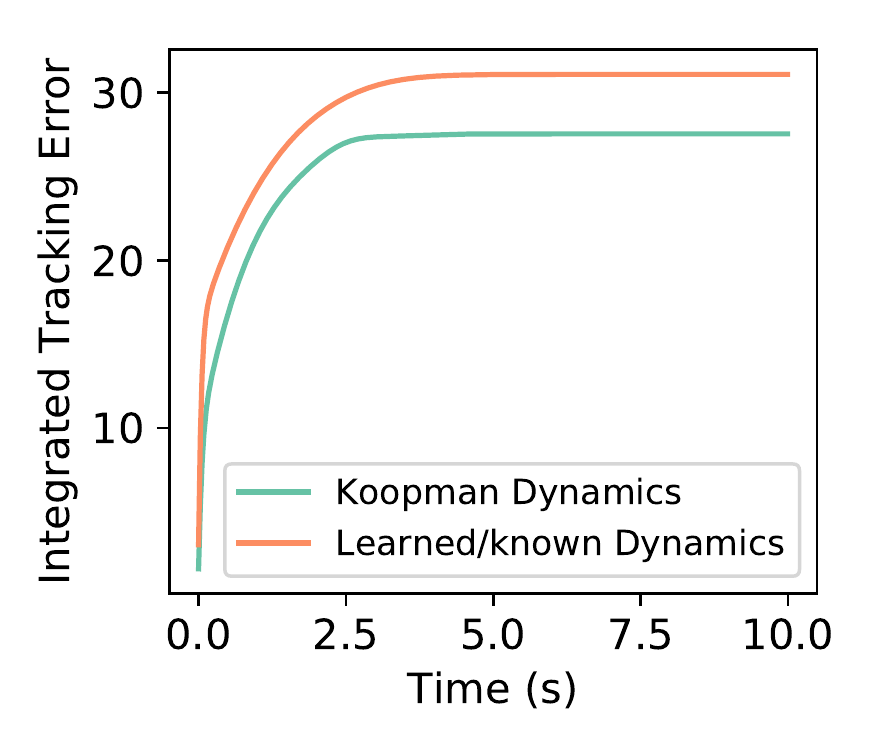}
        \caption{Int. Tracking Error}
        \label{fig:improved_control_error}
      \end{subfigure}
      \begin{subfigure}{0.23\textwidth}
        \centering
        \includegraphics[width=\linewidth]{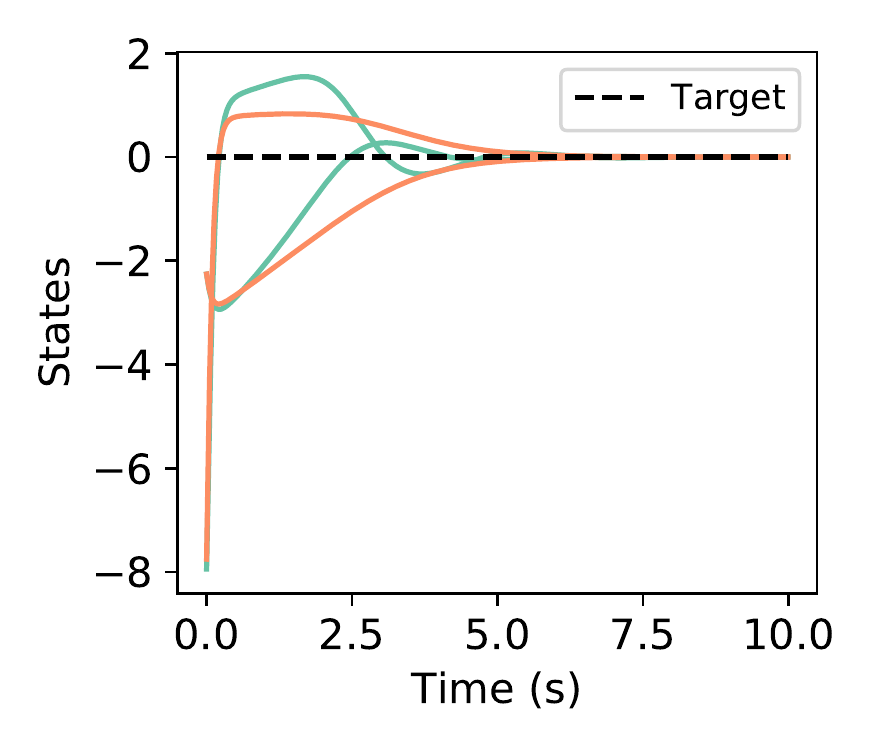}
        \caption{Trajectory}
        \label{fig:improved_control_trajectory}
      \end{subfigure}
      \caption{
                Control performance of a forced van der pol oscillator with an LQR control using the learned Koopman operator,
                the linearization of the known system dynamics, and the linearizion of a learned state-space model using the
                same data and basis functions as the Koopman operator.
                The control performance using the Koopman operator dynamics is shown to outperform the LQR control with known
                dynamics.
                The learned dynamics model performs equally to the known dynamics model and is overlayed on top of the
                known dynamics results.
            }
      \label{fig:ehanced_control_koopman}
    \end{figure}

    Figure~\ref{fig:ehanced_control_koopman} illustrates the improvement in control performance when using the the Koopman
    operator dynamics for LQ control instead of linearizing the dynamics around a local region.
    We compare the control authority using a learned dynamics model in the original state-space using Bayesian optimization with
    the same functions used for the Koopman operator.
    This illustrates that the data used to compute the Koopman operator can learn a nonlinear model of the Van der Pol dynamics
    in the original state-space.
    The Koopman operator formulation of the Van der Pol approximates the dynamic constraints as a
    linear dynamical systems in a higher dimensional space that captures nonlinear dynamical behavior.
    As a result, the Koopman operator formulation coupled with LQ methods can be used to enhance the control the Van der Pol
    system as shown in Figure~\ref{fig:improved_control_trajectory}.
    Computing the resulting trajectory error (Figure~\ref{fig:improved_control_error}) shows that the trajectory taken from the
    Koopman operator controller results in less overall integrated error.
    This is due to formulating the LQ controller with additional information in the form of a dynamical systems that evolves
    functions of state.

    While this example illustrates the possible benefits of utilizing the Koopman operator formulation, we ignored how the data
    was collected for the Van der Pol dynamical system.
    In fact, computing the Koopman operator used random inputs.
    For this example, such an approach works reasonably, but requires a significant amount of data to fully cover the state-space
    of the Van der Pol system.
    The following sections introduce a method that enables a robot to actively learn the Koopman operator.

\section{Control Synthesis for Active Learning of Koopman Operator Dynamics} \label{sec:control_synthesis}
    Active learning controllers need to consider existing polices that solve a task while generalizing to learning objectives.
    In this section, we formulate a controller for active learning that takes into account the Koopman operator dynamics as well
    as polices generated for solving tasks using the Koopman operator linear dynamics.
    We generate an active learning controller that takes into account existing policies by first deriving the
    mode insertion gradient~\cite{egerstedt2003optimal, axelsson2008gradient}.
    The mode insertion gradient calculates how an objective changes when switching from one control strategy to another.
    We then formulate an active learning controller by minimizing the mode insertion gradient
    while including policies that solve a specified task.
    \footnote{During training, the policies derived from the Koopman operator dynamics will be inaccurate; however, over time and gathered experience, both the model and policy will converge. This is a common approach
    in most model-based reinforcement learning techniques~\cite{atkeson1997comparison}.  }
    The derived controller is then shown to increase the rate of change of the information measure,
    which guides the robot towards important regions of state-space, improving the data collection and the quality of
    the learned Koopman operator dynamics model.

    \subsection{Control Formulation}
        Active learning allows a robotic agent to self-excite the dynamical states in order to collect data that results in a
        Koopman operator $\kop$ that can be used describe system evolution.
        We formulate the active learning problem as a hybrid switching problem~\cite{ansari2016sequential} where the goal is to
        switch between a policy for a task to an information maximizing controller that assists the dynamical system in
        collecting informative data.

        Consider a general objective function of the form
        \begin{equation}\label{eq:gen_obj}
            J = \int_{t_i}^{t_i+T} \ell( z(s), \mu(z(s))) ds + m(z(t_i+T))
        \end{equation}
        where  $z(t) : \mathbb{R} \to \mathbb{R}^{c_x}$ is the value of the function observables at time $t$ subject to the
        Koopman dynamics in (\ref{eq:koop_with_control}) starting from initial condition $z(x(t_i))$,
        $\ell(z, u) : \mathbb{R}^{c_x} \times \mathbb{R}^m \to \mathbb{R}$ is the running cost,
        $m(z) : \mathbb{R}^{c_x} \to \mathbb{R}$ is the terminal cost,
        and $\mu(z) : \mathbb{R}^{c_x} \to \mathbb{R}^{c_u}$ is a $\mathcal{C}^1$ differentiable policy.
        In this work, the running cost is split into two parts:
        \begin{equation*}
            \ell(z,u) = \ell_\text{learn}(z,u) + \ell_\text{task}(z,u)
        \end{equation*}
        where $\ell_\text{learn}$ is the information maximizing objective (learning task) and $\ell_\text{task}(z,u)$ is
        the task objective for which the policy $\mu(z)$ is a solution to (\ref{eq:gen_obj}) when $\ell_\text{learn} = 0$.

        Given equation (\ref{eq:gen_obj}), we want to synthesize a controller that is bounded
        to the policy $\mu(z)$, but also allows for improvement of an information measure for active learning.
        To do so, we examine in Proposition~(\ref{prop:1}) how sensitive (\ref{eq:gen_obj}) is to switching
        between the policy $\mu(z)$ to an
        arbitrary control vector $\mu_\star(t)$ at time $\tau$ for a time duration $\lambda$.
        \begin{prop} \label{prop:1}
            The sensitivity of switching from $\mu$ to $\mu_\star$ for all $\tau \in \left[t_i, t_i+T \right]$
            for an infinitesimally small $\lambda$,
            (also known as the mode insertion gradient~\cite{egerstedt2003optimal, axelsson2008gradient}) is given by
            \begin{equation}\label{eq:mode_insertion_gradient}
                \frac{\partial J}{\partial \lambda} \Big\vert_{\tau, \lambda=0} = \rho(\tau)^\top (f_2 - f_1)
            \end{equation}
                where $z(t)$ is a solution to \ref{eq:koop_with_control} with $u(t) = \mu(z(t))$ and $z(t_i) = z(x(t_i))$,
                $f_2 = f(z(\tau), \mu_\star(\tau))$,
                $f_1 = f(z(\tau), \mu(z(\tau)))$, and
            \begin{equation}
            \dot{\rho} =
            -\left(
                \frac{\partial \ell}{\partial z} + \frac{\partial \mu}{\partial z}^\top \frac{\partial \ell}{\partial u}
            \right)
            - \left(
                \frac{\partial f}{\partial z} + \frac{\partial f}{\partial u} \frac{\partial \mu}{\partial z}
            \right) ^\top \rho
            \end{equation}
            subject to the terminal condition $\rho(t_i+T) = \frac{\partial}{\partial z} m(z(t_i+T))$.
        \end{prop}
        \begin{proof}
            See Appendix~\ref{app:proof_prop1}.
        \end{proof}

        We can write an unconstrained optimization problem for calculating $\mu_\star(\tau)$ over the interval
        $\tau \in \left[t_i, t_i+T\right]$ that will minimize the mode insertion gradient.
        We can write this optimization problem using a secondary objective function
        \begin{equation}\label{eq:secondary_objective}
            J_2 = \int_{t_i}^{t_i+T} \frac{\partial J}{\partial \lambda} \Big\vert_{\tau=t, \lambda=0}+ \frac{1}{2}\Vert \mu_\star(t)-\mu(z(t))
            \Vert_{\tilde{\bold{R}}}^2 dt,
        \end{equation}
        where $\tilde{\mathbf{R}} \in \mathbb{R}^{m \times m}$ bounds the change of $\mu_\star$ to $\mu(z)$,
        and $\frac{\partial J}{\partial \lambda}\Big\vert_{\tau=t, \lambda=0}$ is evaluated at $\tau = t$.
        Solving equation (\ref{eq:secondary_objective}) with respect to $\mu_\star(t)$ can be viewed as a functional optimization
        over $\mu_\star(t) \forall t \in \left[t_i, t_i+T \right]$.
        Since equation (\ref{eq:secondary_objective}) is quadratic in $\mu_\star$, we can compute a closed form solution
        for any application time $\tau \in \left[ t_i, t_i + T \right]$.
        \begin{prop}\label{prop:2}
            Assuming that $v(x,u)$ is differentiable, the control solution that minimizes (\ref{eq:secondary_objective}) is
            \begin{equation}\label{eq:sac_control}
                \mu_\star(t) = -\tilde{\bold{R}}^{-1} \left( \kop_u \frac{\partial v}{ \partial u } \right)^\top \rho(t) + \mu(z(t)).
            \end{equation}
        \end{prop}
        \begin{proof}
            Since (\ref{eq:secondary_objective}) is separable in time, we take the derivative of (\ref{eq:secondary_objective})
            with respect to $\mu_\star(t)$ at each point in $t$ which gives the following expression:
            \begin{align}\label{eq:j2dmu}
                \frac{\partial}{\partial \mu_\star} J_2 &= \int_{t_i}^{t_i+T} \frac{\partial}{\partial \mu_\star} \left(
                \rho(t)^\top \left( f_2 - f_1 \right)
                \right) + \tilde{\mathbf{R}} \left( \mu_\star(t) - \mu(z(t) \right) dt \nonumber \\
                & = \int_{t_i}^{t_i+T}  \left( \kop_u \frac{\partial v}{ \partial u} \right)^\top \rho(t) + \tilde{\mathbf{R}} (\mu_\star (t) - \mu(z(t))) dt.
            \end{align}
            Solving for $\mu_\star(t)$ in (\ref{eq:j2dmu}) gives the control solution
            \begin{equation*}
                \mu_\star(t) = -\tilde{\bold{R}}^{-1} \left( \kop_u \frac{\partial v}{ \partial u } \right)^\top \rho(t) + \mu(z(t)).
            \end{equation*}
        \end{proof}

        Proposition~(\ref{prop:2}) gives a formula for switching from $\mu_\star(t)$ to improve the objective~(\ref{eq:gen_obj}).
        We can use equation~(\ref{eq:sac_control}) with (\ref{eq:mode_insertion_gradient}) to show that our approach improves
        the active learning objective subject to bounds placed on arbitrary tasks included in (\ref{eq:gen_obj}).

        \begin{corollary}
            \label{cor:neg_mode_insert}
            Assume that the Koopman operator dynamics for a system are defined by the following control affine structure:
            \begin{equation}\label{eq:control_affine_koopman}
                \dot{z} = \kop_x z(x(t)) + \kop_u v(x(t)) u(t)
            \end{equation}
            where $v(x) : \mathbb{R}^{n} \to \mathbb{R}^{c_u \times m}$.
            \footnote{This formulation assumes that we can recover $x(t)$ from $z(t)$ for computing $v(x)$.}
            Moreover, assume that $\frac{\partial}{\partial \mu} \mathcal{H} \neq 0$ where $\mathcal{H}$ is the control Hamiltonian for (\ref{eq:gen_obj}).
            Then
            \begin{equation} \label{eq:djdlam_neg}
                \frac{\partial}{\partial \lambda} J = -\Vert \left(\kop_u v(x) \right)^\top \rho \Vert_{\tilde{{\mathbf{R}}}^{-1}}^2< 0
            \end{equation}
            for $\mu_\star(t) \in \mathcal{U}$ $\forall t \in \left[t_i, t_i+T\right]$ where $\mathcal{U}$ is the control space.
        \end{corollary}
        \begin{proof}
            Inserting (\ref{eq:sac_control}) into (\ref{eq:mode_insertion_gradient}) gives
            \[
            \frac{\partial}{\partial \lambda} J = \rho(t)^\top \left(\kop_u v(x(t)) \right)\left(
            - \tilde{\mathbf{R}}^{-1} \left( \kop_u v(x(t)) \right) ^\top \rho(t)
            \right)
            \]
            which can be written as the norm
            \[
            \frac{\partial}{\partial \lambda} J = -\Vert \left(\kop_u v(x) \right)^\top \rho \Vert_{\tilde{\mathbf{R}}^{-1}}^2< 0.
            \]
        \end{proof}

        Because we define our objective to be reasonably general, we can add both stabilization terms as well as information measures that allow
        a robot to actively identify its own dynamics.
        The following subsection provides an overview of the Fisher information measure and information bounds based on our controller.
        We first describe the Fisher information matrix for the Koopman operator parameters and then generate an information measure.
        We then show that using (\ref{eq:sac_control}) and Corollary~\ref{cor:neg_mode_insert}, that we can approximately calculate to first order
        the gain in information.

    \subsection{Information Maximization}
        Using the controller defined in (\ref{eq:sac_control}), we investigate information measures that we can use in (\ref{eq:gen_obj}) to enable
        the robot to actively learn the Koopman operator dynamics.
        In this work, we use the Fisher information~\cite{pukelsheim2006optimal,cover2012elements} to generate a information measure for active learning.
        The Fisher information is a way of measuring how much information a random variable has about a set of parameters.
        If we treat calculating the Koopman operator dynamics as a maximum likelihood estimation problem where the likelihood is given by
        $\pi( z \mid \kop) : \mathbb{R}^{c_x} \to \mathbb{R}^+$, we can compute the Fisher information matrix over the parameters
        that compose of the Koopman operator $\kop$.
        The Fisher information matrix is computed as
        \begin{equation}\label{eq:fisher_exp}
            \bold{I} \left[ z \mid \kop \right] =
            \mathbb{E} \left[ \frac{\partial}{\partial \kappa } \log \pi( z \mid \kop) ^\top \frac{\partial }{\partial \kappa} \log \pi (z \mid \kop) \right] \in \mathbb{R}^{\vert \kappa \vert^2}
        \end{equation}
        where $\mathbb{E}$ is the expectation operator, $\kappa = \{ \kop_{i,j} \mid \kop_{i,j} \in \kop \}$,
        and $\vert \kappa \vert$ is the cardinality of the vector $\kappa$.
        Assuming that $\pi$ is a Gaussian distribution, (\ref{eq:fisher_exp}) becomes
        \begin{equation}
            \bold{I} \left[ z \mid \kop \right] = \frac{\partial f}{\partial \kappa}^\top \Sigma^{-1} \frac{\partial f}{\partial \kappa}
        \end{equation}
        where $\Sigma \in \mathbb{R}^{c_x \times c_x}$ is the noise covariance matrix.
        Because the Fisher information defined here is positive semi-definite, we use the trace of the Fisher information
        matrix~\cite{nahi1971design} in $\ell (z, u)$.
        This measure allows us to synthesize control actions that maximize the T-optimality measure of the Fisher information matrix~\cite{nahi1971design}.

        \begin{definition}
            The T-optimality measure is given by the trace of the Fisher information matrix (\ref{eq:fisher_exp}) and defined as
            \begin{equation}\label{eq:t-optimality}
                \mathfrak{I}(\kop) =\text{tr }\mathbf{I}\left[z \mid \kop \right] \ge 0.
            \end{equation}
        \end{definition}
        In this work we incorporate (\ref{eq:t-optimality}) into (\ref{eq:gen_obj}) additively using $1/(\mathfrak{I} + \epsilon)$,
        that is
        \begin{equation*}
            \ell_\text{learn}(z,u) = 1/(\mathfrak{I}(\kop) + \epsilon)
        \end{equation*}
        where $\epsilon \ll 1$ is a small number to prevent singular solutions due to the positive semi-definite Fisher information matrix~\cite{morimura2005utilizing, wei2008dynamics, inoue2003line},
        and $\mathfrak{I}$ is computed using the evaluation of $\kop$ at time $t_i$.
        By minimizing (\ref{eq:gen_obj}) we also minimize the inverse of the T-optimality (which maximizes the T-optimality).

        \begin{assumption}\label{ass:1}
            Assume that $\mathfrak{I}(\tilde{\kop}) > 0$ implies $\mathfrak{I}(\kop)>0$ where $\tilde{\kop}$ is an approximation to the
            Koopman operator $\kop$ computed from the data set $\mathcal{D} = \{ x(t_m), u(t_m)\}_{m=0}^i$ that contains data up until the current sampling time $t_i$.
        \end{assumption}

        \begin{theorem}
            \label{thm:inf}
            Given Assumption~\ref{ass:1} and dynamics~(\ref{eq:control_affine_koopman}), then the change in information
            \footnote{With respect to the information acquired from applying only $\mu(z)$.}
            $\Delta \mathbf{I}$ subject to (\ref{eq:sac_control}) is given to first order
            \begin{multline}
                \Delta \mathbf{I} \approx
                \left(
                    \Vert ( \kop_u v(x) \right)^\top \rho \Vert_{\tilde{\mathbf{R}}^{-1}}^2 + \ell_\text{task}(z, \mu_\star)
                    \\ - \ell_\text{task}(z, \mu)
                )  \mathfrak{I}_{\mu_\star}\mathfrak{I}_{\mu} + \mathcal{O}(\Delta t),
            \end{multline}
            where $\mathfrak{I}_{\mu_\star}$, $\mathfrak{I}_{\mu}$ is the T-optimality measure (\ref{eq:t-optimality}) from applying the control
            $\mu_\star$ and $\mu$.
        \end{theorem}
        \begin{proof}
            See Appendix~\ref{app:proof_thm1}.
        \end{proof}

        Theorem~\ref{thm:inf} shows that our controller increases the rate of information that a robot would have normally
        acquired if it had only used the control policy $\mu(z)$.
        Weighing the information measure against the task objective allows us to ensure that the relative
        information gain is positive when using the active learning controller.
        That is, the difference between the information from using the policy $\mu(x)$ and the control $\mu_\star(t)$ will be positive.
        Other heuristics can be used such as a decaying weight on the information gain or setting the weight to 0
        at a specific time so that the robot attempts the task.
        We provide a basic overview of the control procedure in Algorithm~\ref{alg:active_learning}.
        Videos of the experiments and example code can be found at \url{https://sites.google.com/view/active-learning-koopman-op}.

        \begin{algorithm}
        \caption{Active Learning Control}\label{alg:active_learning}
            \begin{algorithmic}[1]
                \State \textbf{initialize:} objective $\ell(z,u)$, policy $\mu(z)$, normally distributed random $\kop \sim \mathcal{N}(0, \bold{1})$.
                \State sample state measurement $x(t_i)$
                \State add $x(t_i)$ to dataset $\mathcal{D}$, update $\kop$ and $\mu(z)$
                \State simulate $z(t), \rho(t)$ for $t \in \left[ t_i, t_i +T \right]$ with conditions $z(t_i) = z(x(t_i))$ and $\rho(t_i+T) = \frac{\partial}{\partial z} m(z(t_i+T))$ with $\mu(z)$
                \State compute $\mu_\star (t) = - \tilde{\mathbf{R}}^{-1} \left( \kop_u \frac{\partial v}{\partial u}\right)^\top \rho(t) + \mu(z(t))$
                \State \Return $\mu_\star (t_i)$
                \State update timer $t_i \to t_{i+1}$
            \end{algorithmic}
        \end{algorithm}

        The following sections use our derived controller to enable active-learning of Koopman operator dynamics.

\section{Single Execution Active Learning of Free-Falling Quadcopters} \label{sec:single_execution}

    In this example, we illustrate the capabilities of combining the Koopman operator representation of a dynamical systems and active learning for
    single execution model learning of a free-falling quadcopter for stabilization.
    Additionally, we compare our approach to other common learning strategies such as active learning with Gaussian processes
    ~\cite{yan_incremental_sparse_gp, nguyen_NEURO_incremental_sparse_gp, deisenroth2015gaussian}, online model adaptation through direct attempts at
    the tasks of stabilization (common online reinforcement learning and adaptive control approach
    ~\cite{lai1986extended, sin1982stochastic, ding2016recursive, kober_IJRR_reinforcement_survey, kormushev_robotmotorskills_em_rl}),
    and a two-stage noisy motor input (often referred to as ``motor babble''~\cite{saegusa2009active, nagabandi2017neural, reinhart_AuRo_skill_babble}).


    \subsection{Problem Statement}
        The task is as follows: The quadcopter, with dynamics described in Appendix~\ref{app:quad}
        and~\cite{Fan-RSS-16}, must learn a model within the first second of free-falling and then use the model to generate a stabilizing controller,
         preventing itself from falling any further.
        We define success of the quadcopter in the task when $\Vert x - x_d \Vert^2 < 0.01$ where $x_d$ is the desired target state defined by zero
        linear and angular velocity.
        The controllers are designed as linear quadratic regulators using the model that was learned and the LQ objectives provided in
        Section~\ref{sec:enhanced_koopman_control}.
        The parameters used for this example are defined in Appendix~\ref{app:quad} and follows the same parameter choices as in
        Section~\ref{sec:enhanced_koopman_control} for fairness in terms of the learning methods against which we are comparing.

        We compare the information gained (based on the T-optimality condition) and the stabilization error in time against various learning strategies.
        Each learning strategy is tested with the same 20 uniformly sampled initial velocities (and angular velocities) between $-2$ and $2$ radians/meters per second.
        After each trial, the learned dynamics model is reset so that no information from the previous trials are used.

    \subsection{Other Active Learning Strategies}
        We compare our method for active learning against common dynamic model learning strategies.
        Specifically, we compare three model learning approaches against our method, a two-stage noisy control input approach
        ~\cite{saegusa2009active}, a direct stabilization with adaptive model using least squares~\cite{lai1986extended, sin1982stochastic},
        and an active learning strategy using a Gaussian process~\cite{berkenkamp2016safe, schreiter2015safe}.
        Each of these strategies are generating a Koopman operator using the functions of state defined in Appendix~\ref{app:quad} to generate
        a dynamic model of the quadcopter.
        The Gaussian process formulation is the only model where the functions map to the original state-space resulting in a nonlinear dynamics model.

        \paragraph{Least Squares Adaptive Stabilization}
        The first strategy we compare to is to do the task of stabilization at the while updating the model of the dynamics recursively~\cite{lai1986extended, sin1982stochastic}.
        This is often a strategy used in model-based reinforcement learning~\cite{nagabandi2017neural} and adaptive control~\cite{lai1986extended}.
        \paragraph{Two-Stage Motor Babble}
        The second strategy is a two stage approach using noisy motor input (motor babble) for the first second and then pure stabilization~\cite{saegusa2009active}.
        Rather than directly attempting to stabilize the dynamics, the priority is to simply try all possible motor inputs regardless of the model of the dynamics that is being constructed.
        The motor babble strategy allows us to bound the motor excitation which prevents the rotor from destabilizing once the learning stage is complete.
        As with the direct stabilization method, we use a recursive least squares to update the model of the Koopman operator.
        \paragraph{Active Learning with Gaussian Process}
        The last strategy is an active Gaussian process strategy~\cite{berkenkamp2016safe, schreiter2015safe}.
        In this active learning strategy, we build a model of the dynamics of the quadcopter by generating a Gaussian process dynamics model~\cite{berkenkamp2016safe, deisenroth2015gaussian}.
        Using the variance estimate~\cite{schreiter2015safe}, we uniformly sample points around the current state bounded by some $\epsilon$ constant and find the state which maximizes the variance.
        The sampled state with the largest variance is then used to generate a local LQ controller to guide the quadcopter dynamics to that state to collect the data.
        After the first second, the Gaussian process model is used to generate a stabilizing controller by linearizing the model about the final desired stabilization state.
        The kernel function used is computed using the functions of state provided in Appendix~\ref{app:quad} for a fair comparison.

        Note that for the two-stage, least squared adaptive, and our approach, we learn a Koopman operator dynamics model which we use to compute an LQ controller.
        The Gaussian process model is in in the original state-space as described in~\cite{deisenroth2015gaussian}.

        \begin{figure*}[h!]
            \centering
            \begin{subfigure}{0.3\textwidth}
                \centering
                \includegraphics[width=\linewidth]{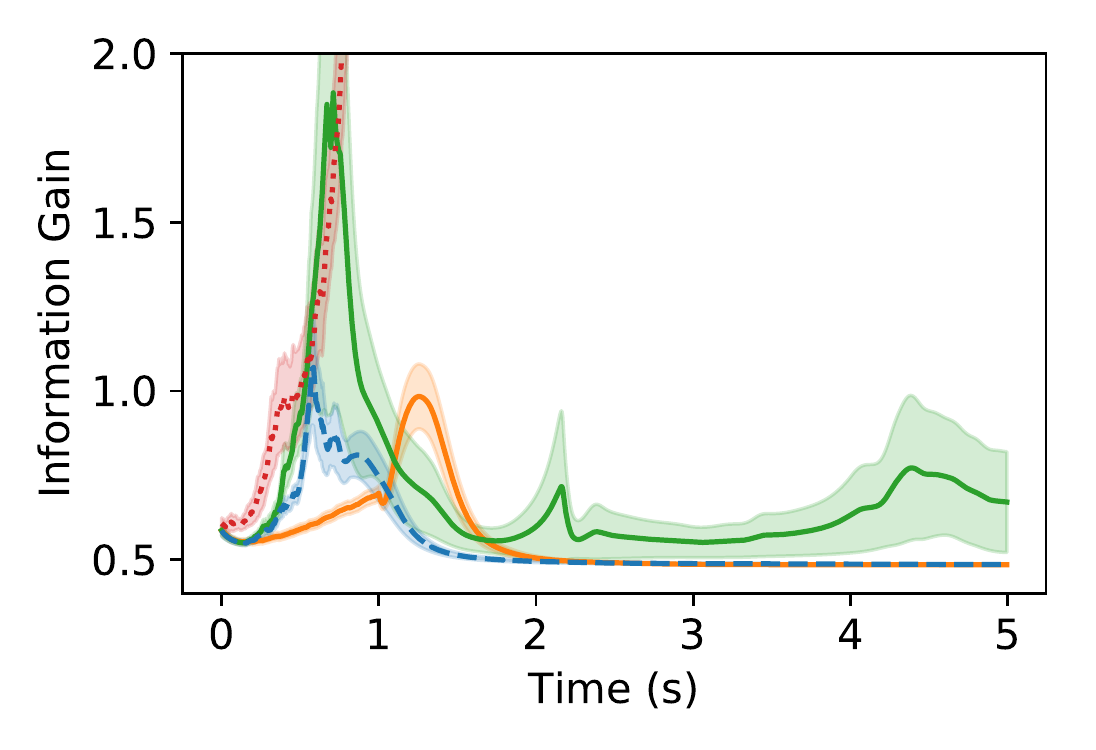}
                \caption{Inf. Gain}
            \end{subfigure}
            \begin{subfigure}{0.3\textwidth}
                \centering
                \includegraphics[width=\linewidth]{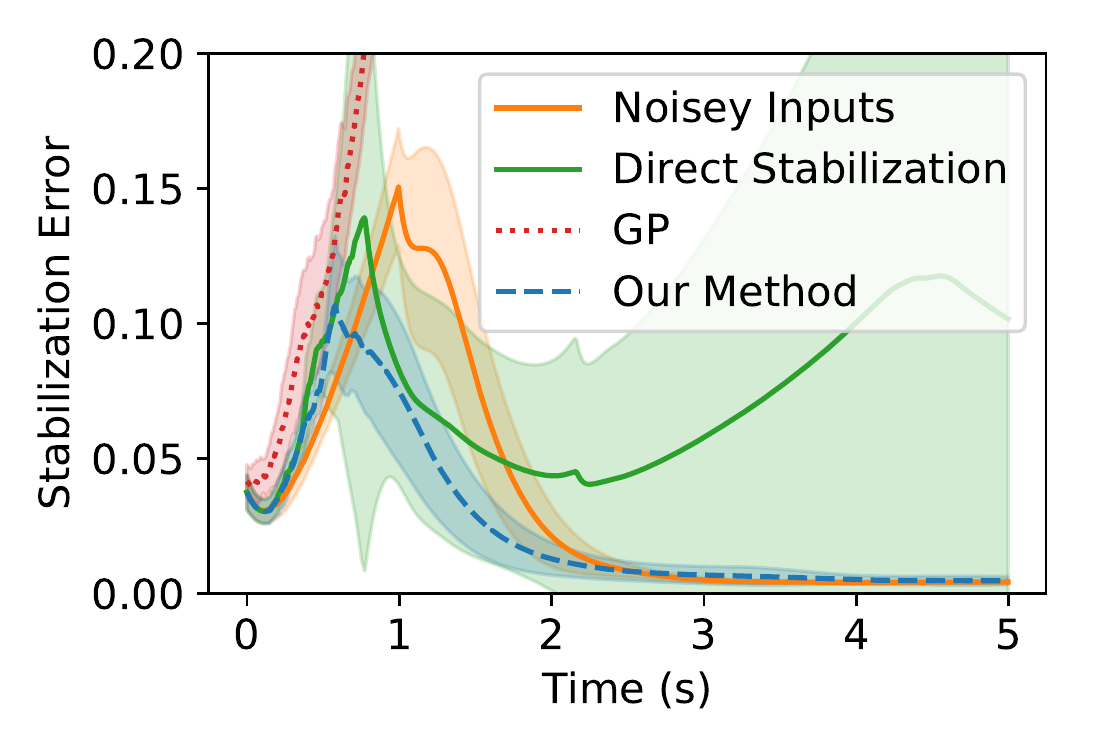}
                \caption{Stabilization Error}
            \end{subfigure}
            \begin{subfigure}{0.3\linewidth}
                \centering
                \includegraphics[width=0.6\linewidth]{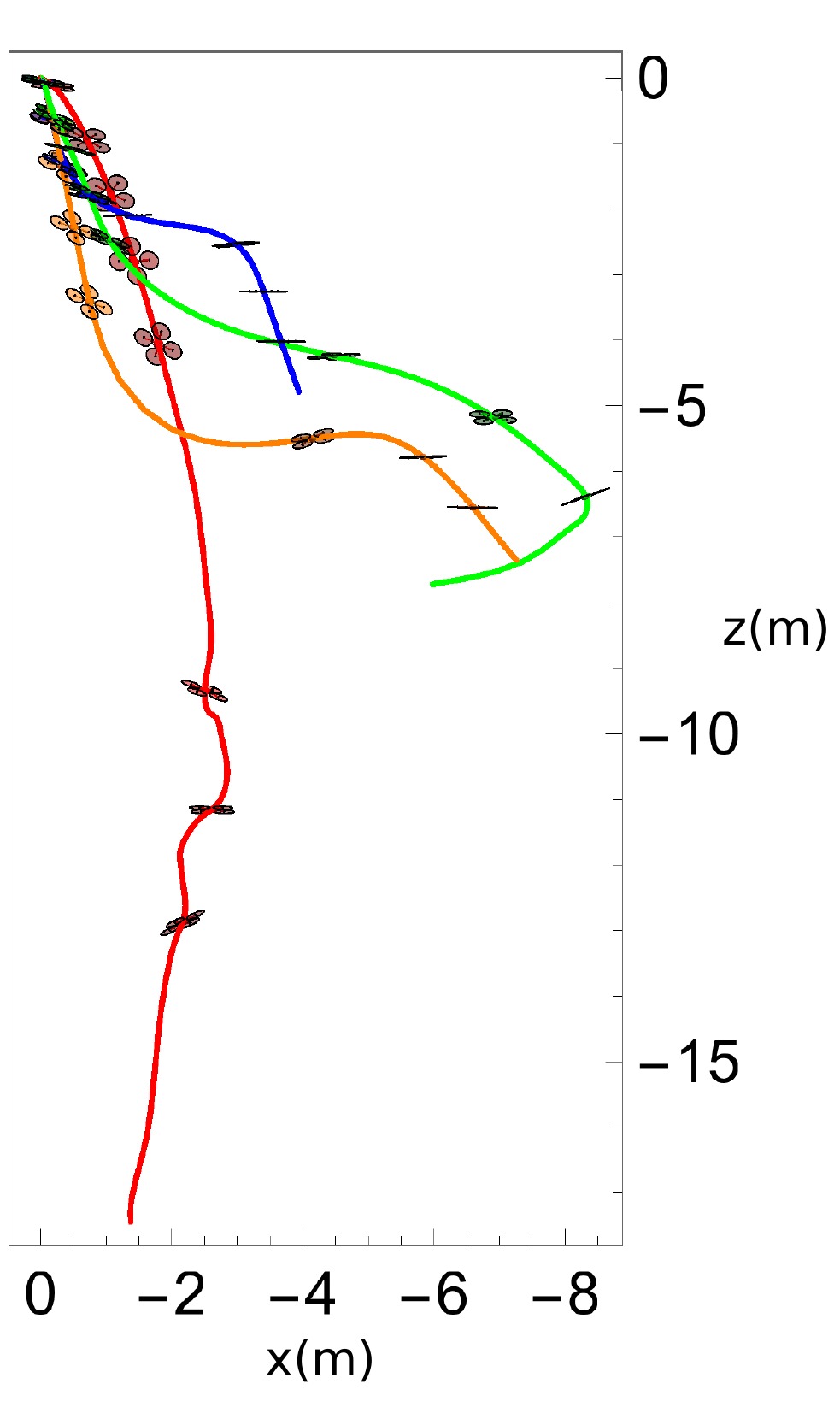}
                \caption{Sample Trajectories}
            \end{subfigure}
            \caption{
                Monte-Carlo simulation comparing various learning strategies to stabilize a quadcopter falling for $20$ trials with uniformly sampled initial linear and angular velocities.
                (a) Information gain (trace of the Fisher information matrix) is shown for the various learning strategies.
                (b) Stabilization error and standard deviation is shown over time for each learning strategy over $20$ trajectories.
                (c) Representative time series snapshots are shown depicting the various learning strategies.
                With our approach, maximization of the information measure, coupled with the Koopman operator formulation of the dynamics, enables quick stabilization of the quadcopter.
            }
            \label{fig:mc_recovery}
        \end{figure*}

    \subsection{Results}

        Figure~\ref{fig:mc_recovery} (a) illustrates the information (T-optimality of the Fisher information matrix) for each method.
        Our approach to active learning is shown to improve upon the information when compared to motor babble
        (the most basic method for active learning).
        The other methods outperform our approach in terms of the overall information gain by overly exciting the dynamics.
        The direct adaptive stabilization method utilizes the incorrect dynamics model to self-adjust and eventually stabilize
        the quadcopter (as shown in the variance).
        The active Gaussian process approach uses the covariance estimate to actuate the quadcopter towards uncertain regions.
        Collecting data in uncertain regions allows the active Gaussian process approach to actively select where the quadcopter
        should collect data next.

        It is worth noting that these approaches will often lead the quadcopter towards unstable regions, making it difficult
        to stabilize the dynamics in time.
        Our approach actively synthesizes when it is best to learn and stabilize which assists in quickly stabilizing the
        quadcopter dynamics (see Figure~\ref{fig:mc_recovery} (b)).
        The addition of the Koopman operator dynamics further enhances the control authority of the quadcopter as shown with the
        direct adaptive stabilization, motor babble, and our approach to active learning.
        While the active Gaussian process model does at times succeed, the method relies on both the quality of data acquired and
        the local linear approximation to the dynamics.
        This results in a deficit of nonlinear information that is needed to successfully achieve the learning task in a
        single execution.

    \subsection{Sensitivity to Initialization and Parameters}

        We further test our algorithm against sensitivities to initialization of the Koopman operator.
        Our algorithm requires an initial guess at the Koopman operator in order to boot-strap the active learning process.
        We accomplish this using the same experiment described in the previous section which used a zero mean, variance of $1$
        normally distributed initialization of the Koopman operator.
        We vary the variance that initializes the Koopman operator parameters using a normal distribution with
        zero mean and a variance experiment set of $\{0.01, 0.1, 1.0, 10.0\}$.

        In Fig.~\ref{fig:koopman-init-variations} we find that so long as the initialization of the Koopman operator is within a
        reasonable initialization (non-zero and within an order of magnitude), the performance is comparable to active learning
        described in Fig.~\ref{fig:mc_recovery}.
        However, this may not be true for all autonomous systems and results may vary depending on the sampling frequency and
        the behavior of the underlying system.
        A benchmark is provided for stabilizing the quadcopter when the Koopman operator is precomputed in
        Fig~\ref{fig:koopman-init-variations} illustrating the performance of the control authority when using the Koopman
        operator-based controller.

        \begin{figure}
            \centering
            \begin{subfigure}{0.23\textwidth}
                \includegraphics[width=\linewidth]{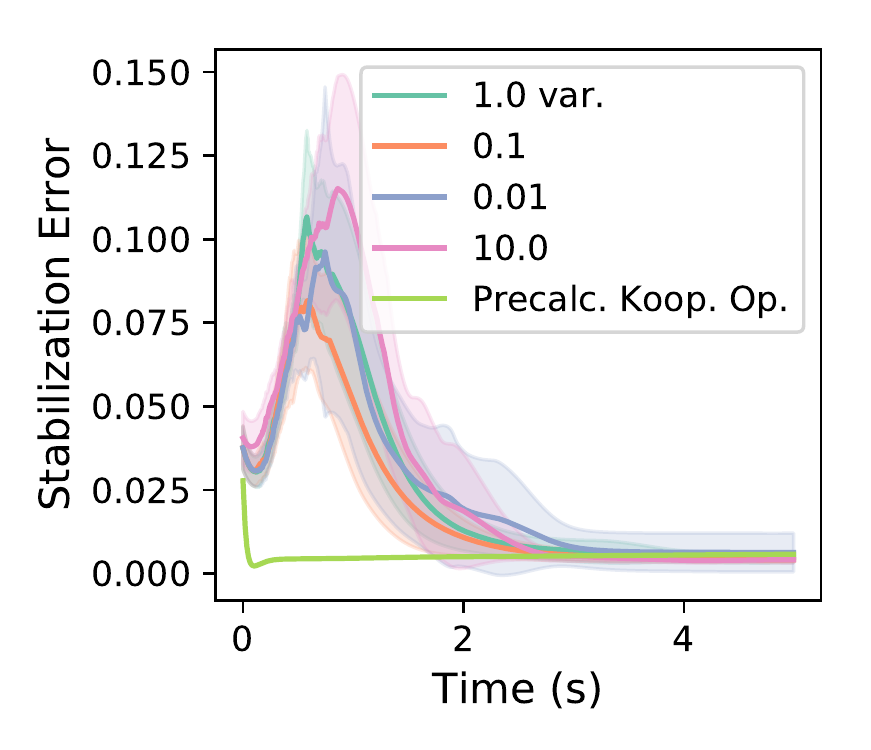}
                \caption{Stab. Err.}
            \end{subfigure}
            \begin{subfigure}{0.23\textwidth}
                \includegraphics[width=\linewidth]{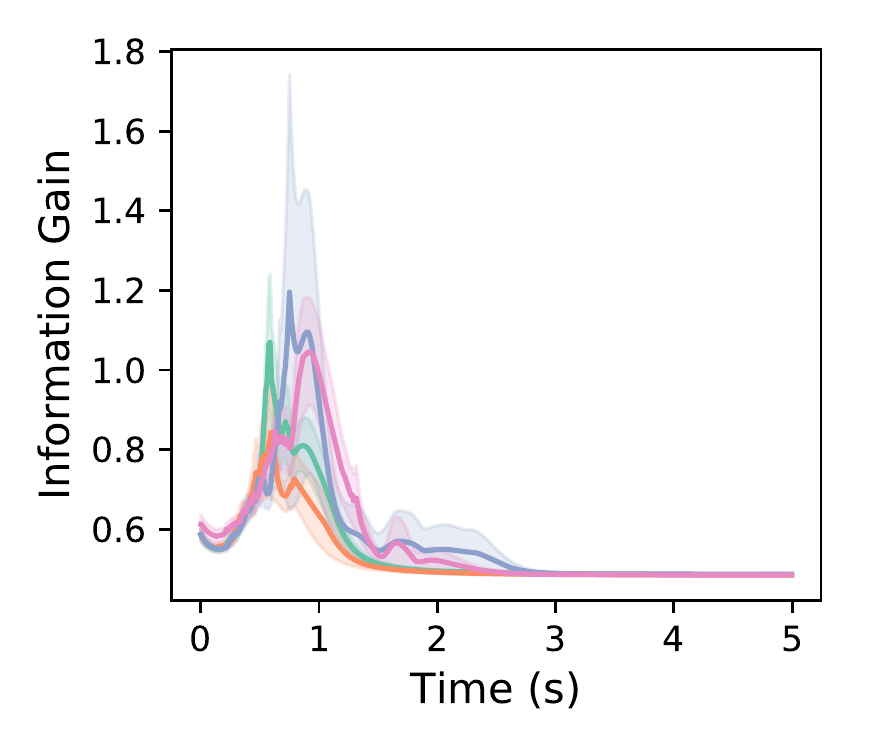}
                \caption{Inf. Gain}
            \end{subfigure}
            \caption{
                    Resulting sensitivities in stabilization error and information gain with respect to variance levels in
                    Koopman operator initialization.
                    Benchmark stabilization performance is provided for known/precalculated Koopman operator.
            }
            \label{fig:koopman-init-variations}
        \end{figure}

        The choices in the parameters of our algorithm can also effect its performance.
        Specifically, setting the value of the regularization term $\tilde{\mathbf{R}}$ too large will prevent the robot from
        significantly exploring the states of the robot.
        In contrast, if the regularization term is set too low, the robot will widen its breath of exploration which can be
        harmful to the robot if the states are not bounded.
        A similar effect is achieved by adding a weight on the active learning objective.

        Changes in the time horizon $T$ will also effect the performance of the algorithm.
        Generally, smaller $T$ will result in more reactive behaviors where larger $T$ tends to have more intent driven
        control responses.
        Choosing these values appropriately will be problem specific; however, the limited number of tunable parameters
        (not including choosing a task objective) provides the advantage of ease of implementation.

    \subsection{Discussion}

        While the single execution capabilities of the Koopman operator with active learning is appealing,
        not all robotic systems will be capable of such drastic performance.
        In particular, this example relies on some prior knowledge of the underlying robotic system and the dynamics
        that govern the system.
        The functions of state are chosen such that they include nonlinear elements (e.g, cross product terms that we
        expect will help in stabilization).
        Thus, the approximate Koopman operator is predicting the evolution of nonlinear elements found in the original
        nonlinear dynamics.
        Often these underlying structures that we can exploit are not known or easily found in robotics.
        Choosing random polynomial or Fourier expansions as function observables can sometimes work
        (see Section~\ref{sec:robot_experiments}),
        but often can lead to unstable eigenvalues in the Koopman operator dynamics which can make model-based control difficult
        to synthesize~\cite{brunton2016koopman}.

        Recent work has attempted to address these issues using sparse optimization~\cite{kramer2017sparse} or discovering invariances in the state-space~\cite{brunton2016koopman}.
        A promising method is automating the discovery of the function observables by learning the functions from data~\cite{yeung2017learning}.
        By using current advances in neural networks and function representation, it is possible to automate the discovery of function observables.
        The following section further develops the work in automating the discovery of function observables for Koopman operators through the use of our approach for active learning.

\section{Automating Discovery of Koopman operator Function Observables} \label{sec:automating_function_discovery}

    As a solution to automating the choice of function observables, the use of deep neural networks~\cite{yeung2017learning} have been used to automatically discover the function observables.
    In this section, we illustrate that we can use these neural networks coupled with our approach for active learning to automatically discover the Koopman operator and the associated functions of state.

    \begin{figure}[h!]
      \centering
      \begin{subfigure}{0.22\textwidth}
        \centering
        \includegraphics[width=\linewidth]{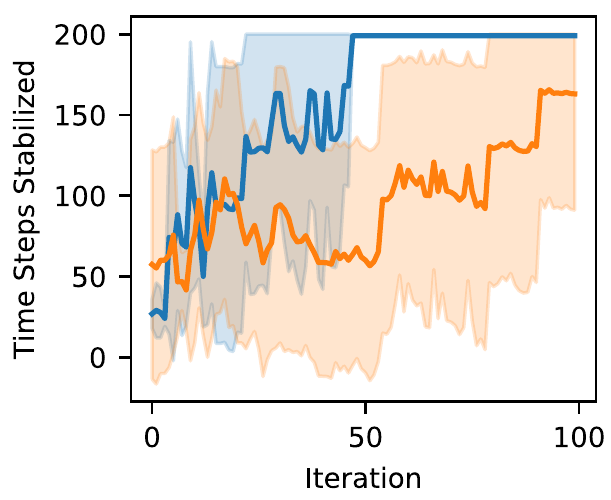}
        \caption{Cart Pendulum}
        \label{fig:deep_cp}
      \end{subfigure}
      \begin{subfigure}{0.22\textwidth}
        \centering
        \includegraphics[width=\linewidth]{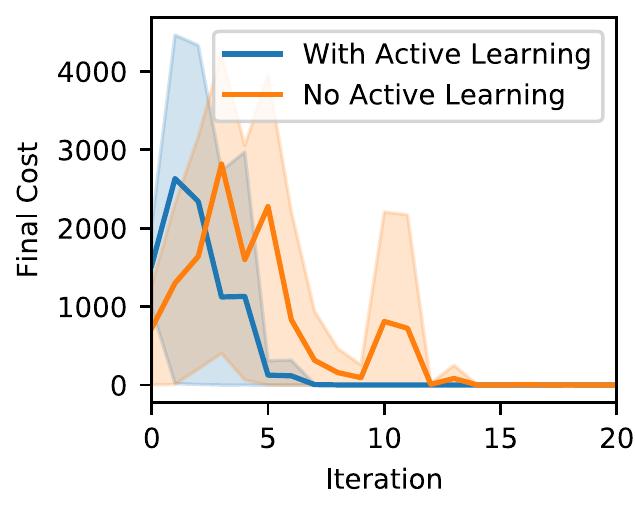}
        \caption{2-Link Robot}
        \label{fig:deep_2l}
      \end{subfigure}
      \caption{(a) Resulting stabilization time of a cart pendulum using Koopman operators with automatic function discovery.
      (b) Control response of a 2-link robot using Koopman operators with automatic function discovery. Active learning improves the rate of success of each task.  }
      \label{fig:deep_nn}
    \end{figure}

    \subsection{Including Automatic Function Discovery}
        Revisiting Equation~\ref{eq:koop_with_control}, we can parameterize $z(x)$ and $v(x,u)$ using a multi-layer neural network with parameters $\theta \in \mathbb{R}^d$.
        We denote the parameterization of $z,v$ as $z_\theta(x)$ and $v_\theta(x,u)$ where the subscript $\theta$ denotes
        the function observables are parameterized by the same set of parameters $\theta$.
        Given the same data set that was defined previously, $\mathcal{D} = \{x(t_m), u(t_m) \}_{m=0}^{M}$, the new optimization problem that is to be solved is
        \begin{equation}\label{eq:nn_regression}
            \min_{\kop, \theta} \frac{1}{2} \sum_{m=0}^{M-1} \Vert \tilde{z}_\theta (x(t_{m+1}),u(t_{m+1}) ) - \kop \tilde{z}_\theta (x(t_{m}),u(t_{m})) \Vert^2,
        \end{equation}
        where $\tilde{z}_\theta (x,u)= \left[ z_\theta(x)^\top, v_\theta(x,u)^\top \right]^\top$.
        Equation (\ref{eq:nn_regression}) can be solved using any of the current techniques for gradient descent (Adams method~\cite{kingma2014adam} is used in this work).
        The continuous time Koopman operator is obtained similarly using the matrix log of $\kop$, resulting in the differential equation
        \begin{equation}\label{eq:nn_koop_dynamics}
            \dot{z}_\theta = \kop_x z_\theta(x(t)) + \kop_u v_\theta(x(t), u(t)).
        \end{equation}
        Because we are now optimizing over $\theta$, we lose the sample efficiency of single execution learning that was illustrated in the example in Section~\ref{sec:single_execution}.
        Active learning can be used; however, adding the additional parameters $\theta$ to the information measure significantly
        increases the computational cost of calculating the Fisher information measure (\ref{eq:fisher_exp}).
        As a result, we only compute the information measure with respect to $\kop$ in order to avoid the computational overhead of maximizing information with respect to $\theta$.

    \subsection{Examples}

        We illustrate the use of deep networks for automating the function observables for the Koopman operator for stabilizing a cart pendulum and controlling a 2-link robot arm to a target.
        A neural network is first initialized (see Appendix~\ref{app:deep_nn} for details) for the Koopman operator functions $z_\theta, v_\theta$ as well as an  LQ controller for the task at hand.
        At each iteration, the robot attempts the task and learns the Koopman operator dynamics by minimizing (\ref{eq:nn_regression}).
        We compare against decaying additive control noise as well as our method for active learning where a weight on information measure is used which decays at each iteration according to  $\gamma^{i+1}$ where $0 < \gamma < 1$ and $i$ is the iteration number.
        The data collected is then used to update the parameters $\theta$ and $\kop$ using (\ref{eq:nn_regression}) and the LQ controller is updated with the new $\kop_x, \kop_u$ parameters.

        Figure~\ref{fig:deep_nn} illustrates that we can automate the process of learning the function observables as well as the Koopman operator.
        With the addition of active learning, the process of learning the Koopman operator and the function observables is improved.
        In particular, stabilization of the cart pendulum is achieved in only $50$ iterations in comparison to additive noise which takes over $100$ iterations.
        Similarly, the 2-link robot can be controlled to the target configuration within $5$ iterations with our active learning approach.

    \subsection{Discussion}
        While this method is promising, there still exist significant issues that merit more investigation in future work.
        One of which is the trivial solution where $z_\theta, v_\theta = 0$.
        This issue often occurs with how the parameters $\theta$ were initialized.
        This trivial solution has been addressed in~\cite{lusch2017deep}; however, their approach requires significantly complicating how the  regression (\ref{eq:nn_regression}) is formulated.
        We found that adding the state $x$ as part of the neural network output of $z_\theta$ was enough to overcome the trivial solution.

\section{Robot Experiments} \label{sec:robot_experiments}

    Our last set of examples test our active learning strategy with robot experiments.
    We use the robots depicted in Figure~\ref{fig:robots} to illustrate control and active learning with Koopman operators.
    The sphero SPRK robot (Figure~\ref{fig:SPRK}) is a differential drive robot inside of a clear outer ball.
    We test trajectory tracking of the SPRK robot in a sand terrain where the challenge is that the SPRK must be able to learn how to maneuver in sand.
    The Sawyer robot (Figure~\ref{fig:sawyer}) is a 7-link robot arm whose task is to track a trajectory defined at the end effector where the challenge is the high dimensionality of the robot.
    We refer the reader to the attached multimedia which has clips of the experiments.

    \begin{figure}[h!]
      \centering
      \begin{subfigure}[b]{0.22\textwidth}
        \centering
        \includegraphics[width=0.9\linewidth]{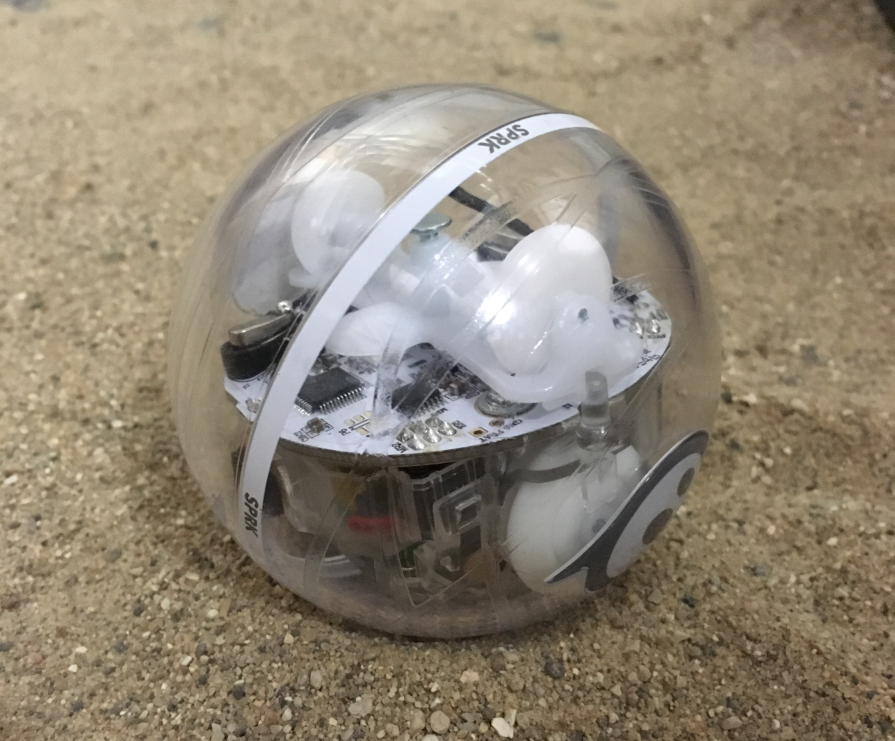}
        \caption{Sphero SPRK}
        \label{fig:SPRK}
      \end{subfigure}
      \begin{subfigure}[b]{0.22\textwidth}
        \centering
        \includegraphics[width=\linewidth]{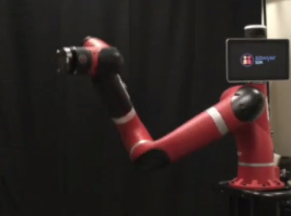}
        \caption{Sawyer Robot}
        \label{fig:sawyer}
      \end{subfigure}
      \caption{Depiction of robots used for experimentation.}
      \label{fig:robots}
    \end{figure}

    \subsection{Experiments: Granular Media and Sphero SPRK}

        Active learning is applied in an experimental setting using the Sphero SPRK robot (Fig.~\ref{fig:SPRK}) in sand.
        The interaction between sand and the SPRK robot makes physics-based models challenging.

        \begin{figure}[h]
          \centering
          \begin{subtable}[b]{0.22\textwidth}
            \centering
            \includegraphics[width=\linewidth]{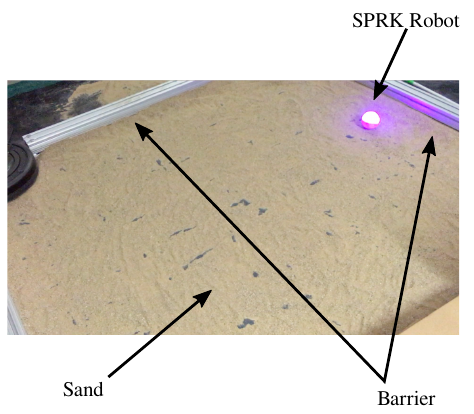}
            \caption{Experimental Setup}
            \label{fig:sand_exp_setup}
          \end{subtable}
          \begin{subtable}[b]{0.22\textwidth}
            \centering
            \includegraphics[width=1.2\linewidth]{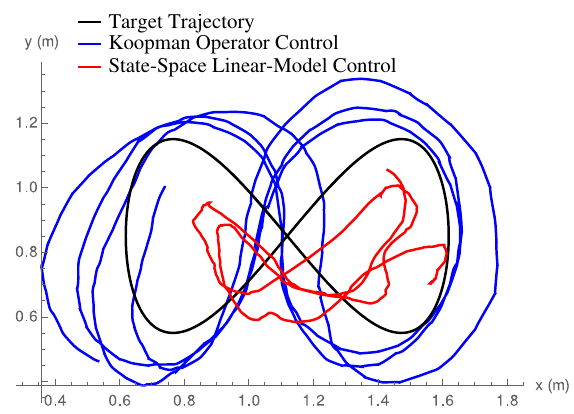}
            \caption{SPRK Trajectories}
            \label{fig:sand_exp_trajectories}
          \end{subtable}
          \\
          \hspace{10mm}
          \begin{subtable}[t]{0.5\textwidth}
          \centering
            \small\addtolength{\tabcolsep}{-3pt}
          \begin{tabular}{|c | c | c | c|}
            \hline
            Method & RMSE & Correlation & Phase Lag (rad)\\
          \hline
            Koopman-based Control & 0.3010 & 0.4028 & 1.1262 \\
            Controller in~\cite{Abraham-RSS-17} & 0.3535 & 0.1034 &  1.4667\\
            \hline
            \end{tabular}
            \caption{Controller performance}
            \label{fig:sand_exp_results}
          \end{subtable}
          \caption{
                Experiment using the Sphero SPRK robot in sand. (a) The experimental setup is depicted with the SPRK robot inside the sand pit.
                Position information is calculated with an overhanging Xbox Kinect using OpenCV~\cite{itseez2015opencv} for tracking.
                (b) Performance of the SPRK robot using the Koopman operator-based controller after active learning.
                Performance is compared with results from~\cite{Abraham-RSS-17}.
                (c) Performance measures showing active learning significantly outperforms non-active learning in robot experiment.
                 The attached multimedia shows the experiment executed.}
          \label{fig:sand_exp}
        \end{figure}

        The parameters for the experiment are defined in Appendix~\ref{app:SPRK}.
        The experiment starts with $20$ seconds of active learning.
        After actively identifying the Koopman operator, the weight on information maximizing is set to zero at $t=20$ and the
        objective is switched to track the trajectory shown in Fig.~\ref{fig:sand_exp_trajectories}.
        In Fig.~\ref{fig:sand_exp_results}, we show the average root mean squared error (RMSE) of the $x-y$ trajectory tracking,
        the average $x-y$ Pearson's correlation using a two-sided hypothesis testing (values close to $1$ indicate
        responsive controllers), and the phase lag of the experimental results.
        Note that in contrast to previous work by the authors~\cite{Abraham-RSS-17}, the method of actively learning the Koopman
        operator improves the performance of the model-based controller.
        In particular, we find that the overall responsiveness and phase lag of the Koopman-based controller improved
        after active learning in sand.

        \begin{figure}[h]
          \centering
          \begin{subtable}[b]{0.22\textwidth}
            \centering
            \includegraphics[width=\linewidth]{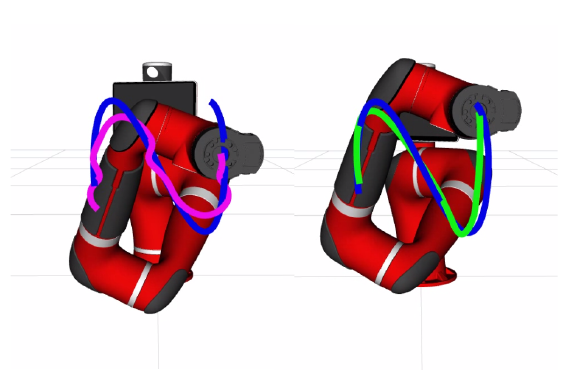}
            \caption{Experimental Visualization}
            \label{fig:sawyer_exp_setup}
          \end{subtable}
          \begin{subtable}[b]{0.22\textwidth}
            \centering
            \includegraphics[width=1.\linewidth]{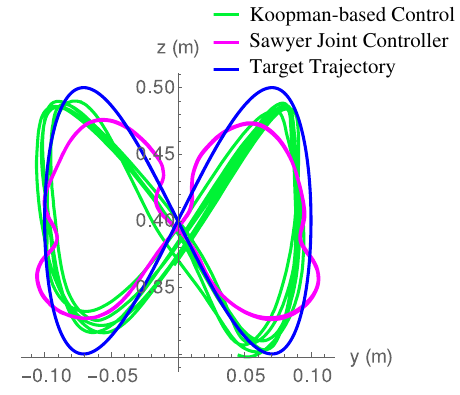}
            \caption{Sawyer Trajectories}
            \label{fig:sawyer_exp_trajectories}
          \end{subtable}
          \\
          \hspace{10mm}
          \begin{subtable}[t]{0.5\textwidth}
          \small\addtolength{\tabcolsep}{-2pt}
          \begin{tabular}{|c | c | c | c|}
          	\hline
             Method & RMSE & Correlation & Phase Lag (rad) \\
        	\hline
            Koopman-based Control & 0.0228 & 0.9777 & 0.2826 \\
            Sawyer Joint Controller & 0.0443 & 0.6026 & 0.7041\\
           	\hline
            \end{tabular}
            \caption{Controller performance}
            \label{fig:sawyer_exp_results}
          \end{subtable}
          \caption{
                    Experiment using Sawyer.
                    Experimental data visualized using RViz~\cite{quigley2009ros}.
                    (a) End-effector trajectory paths using
                    the embedded Rethink Joint controller and Koopman operator controller.
                    Both controllers are running at 100 Hz.
                    (b) Trajectory overlaid from both controller responses.
                    (c) Controller performance shows that active learning for Koopman operator-based controllers performs comparably.
                    We refer the reader to the attached multimedia to view clips of this experiment.}
          \label{fig:sawyer_exp}
        \end{figure}

    \subsection{Experiments: Trajectory Tracking of Rethink Sawyer Robot}
        In this experiment, we use active learning with the Koopman operator to model a $7$ DoF Sawyer robot arm from Rethink Robotics.
        The 7-DoF system is of interest because it is both high dimensional and inertial effects tend to dominate the dynamics
        of the system.
        We define the parameters used for this experiment in Appendix~\ref{app:sawyer}.

        Figure~\ref{fig:sawyer_exp} illustrates a comparison of the embedded controller in the Sawyer robot and the data-driven
        Koopman operator controller.
        Here, we show the average root mean squared error of the tracking position, the Pearson's correlation using a two-sided
        hypothesis testing (values close to $1$ indicate responsive controllers), and the phase lag of the trajectory tracking.
        The resulting controller using the Koopman operator is shown to be comparable to the built-in controller with the
        inclusion of the evolution of the nonlinearities on the Sawyer robot which improve overall trajectory tracking performance.
        The trajectories of the two methods are overlaid which illustrates the improvement in control from the Koopman operator
        after active learning has occurred.
        Since data is always being acquired online, the Koopman operator is continuously being updated as the robot is tracking
        the trajectory.
        The Koopman operator-based controller is able to capture dynamic effects of the individual joints from data.
        This is further reinforced by the improved results found
        Note that one can build a model to solve for similar, if not better, inverse dynamics of the Sawyer robot that can be
        computed for control.
        In particular, the Sawyer robot provides an implementation of inverse dynamics in the robot's embedded controller.
        However, our approach provides high accuracy without needing such a model ahead of time and without linearizing
        the nonlinear dynamics.

\section{Conclusion} \label{sec:conclusion}

    In this paper, we use Koopman operators as a method for enhancing control of robotic systems.
    In addition, we contribute a method for active learning of Koopman operator dynamics for robotic systems.
    The active learning controller enables the robots to learn their own dynamics quickly while taking into account the linear structure of the Koopman operator to enhance LQ control.
    We illustrate various examples of robot control with Koopman operators and provide examples for automating design choices for Koopman operators.
    Last, we show that our method is applicable to actual robotic systems.


\section*{Acknowledgment}
The authors would like to thank Giorgos Mamakoukas for his insight and thorough review of this paper.

This material is based upon work supported by the National Science Foundation under awards
NSF CPS 1837515. Any opinions, findings, and conclusions
or recommendations expressed in this material are those of the author(s) and
do not necessarily reflect the views of the National Science Foundation

\appendices
\section{Parameters for Various Examples }

    \subsection{Control of forced van der pol oscillator}
    \label{app:vdp_dynamics}
        The nonlinear dynamics that govern the Van der Pol oscillator are given by the differential equations
        \begin{equation*}
        \frac{d}{dt}\begin{bmatrix}
        x_1 \\
        x_2
        \end{bmatrix} =
        \begin{bmatrix}
        x_2 \\
        -x_1 + \epsilon (1 - x_1^2) x_2 + u
        \end{bmatrix}
        \end{equation*}
        where $\epsilon = 1$ and $u$ is the control input.

        The Koopman operator functions used are defined as
        \begin{equation*}
        z(x) = \left[x_1, x_2, x_1^2, x_2 x_1^2 \right]^\top
        \end{equation*}
        and $v(u) = u$.
        The same functions are used to compute a regression problem where the final equation is given by
        \begin{equation*}
        \frac{d}{dt}
        \begin{bmatrix}
            x_1 \\
            x_2
        \end{bmatrix} =
        \mathbf{A} z(x) + \mathbf{B} v(u)
        \end{equation*}
        where $\mathbf{A} \in \mathbb{R}^{n \times c_x}$ and $\mathbf{B} \in \mathbb{R}^{n \times c_u}$ are both generated using linear regression.

        The weight parameters for LQ control are
        \begin{equation*}
        \mathbf{Q} = \text{diag} \left(\left[1, 1\right]\right) \text{ and } \mathbf{R} = 0.1
        \end{equation*}
        where
        \begin{equation}\label{eq:expanded_weights}
            \tilde{\mathbf{Q}} = \begin{bmatrix}
            \mathbf{Q} & \bold{0} \\
            \bold{0} & 0
            \end{bmatrix} \in \mathbb{R}^{c_x \times c_x}
        \end{equation}

    \subsection{Quadcopter Free-Falling}
    \label{app:quad}
        The quadcopter system dynamics are defined as
        \begin{eqnarray*}
        \dot{h} &=& h \begin{bmatrix}
        \hat{\omega} & v \\
        \bold{0} & 0
        \end{bmatrix}, \\
        J \dot{\omega} &=& M + J \omega \times \omega, \\
        \dot{v} &=& \frac{1}{m}F e_3 - \omega \times v - g R^T e_3,
        \end{eqnarray*}
        where $h = (R, p) \in \text{SE}(3)$, the inputs to the system are $u =\left[ u_1,u_2,u_3,u_4 \right]$, and
        \begin{align*}
        F = k_t (u_1 + u_2 + u_3 + u_4), \\
        M = \begin{bmatrix}
        k_t l (u_2 - u_4) \\
        k_t l (u_3 - u_1) \\
        k_m (u_1 - u_2 + u_3 - u_4)
        \end{bmatrix}
        \end{align*}
        (see \cite{Fan-RSS-16} for more details on the dynamics and parameters used).
        Note that in this formulation of the quadcopter, the control vector $u$ has bidirectional thrust.

        The measurements of the state of the quadcopter are given by
        \begin{equation}
        \left[ a_g, \omega, v\right]^\top \in \mathbb{R}^9
        \end{equation}
        where $a_g\in\mathbb{R}^3$ denotes the body-centered gravity vector and $\omega, v$ are the body angular and linear velocities respectively.
        The sampling rate for this system is $200$ Hz.

        We define the basis functions for this system as
        \begin{equation*}
        z(x) = \left[ a_g, \omega, v, g(v, \omega) \right]^T \in \mathbb{R}^{18}
        \end{equation*}
        where
        $g(v, \omega) = [ v_3\omega_3, v_2\omega_3, v_3 \omega_1, v_1 \omega_3, v_2 \omega_1,$ $v_1 \omega_2, \omega_2 \omega_3, \omega_1 \omega_3, \omega_1 \omega_2 ]
        $ are the chosen basis functions such that $\omega_i, v_i$ are elements of the body-centered angular and linear velocity $\omega,v$ respectively.
        The functions for control are
        \begin{equation*}
        v(u) = u \in \mathbb{R}^4.
        \end{equation*}

        The LQ control parameters for the stabilization problem are given as
        \[
        \mathbf{Q} = \text{diag}\left( \left[
        1,1,1,1,1,1,5,5,5
        \right] \right) \text{ and }
        \mathbf{R} = \text{diag} \left( \left[
        1, 1, 1, 1
        \right] \right)
        \]
        where the weight on the additional functions $\tilde{\mathbf{Q}}$ are set to zero as in  (\ref{eq:expanded_weights}) .
        The time horizon used in $0.1 s$.

        The active learning controller uses a weight on the information measure of $0.1$ and a regularization weight $\tilde{\mathbf{R}} = \text{diag}( 1000, 1000, 1000, 1000])$.
        Motor noise used in the two-stage method is given by uniform noise at $33\%$ of the control saturation.

    \subsection{Neural Network Automatic Function Discovery Configuration}
    \label{app:deep_nn}

        In this example, we use the Roboschool environments~\cite{klimov2017roboschool} for the robot simulations.

        For the cart pendulum example, we use a three layer network with a single hidden layer for $z_\theta$ and $v_\theta$ with $\{4, 20, 40 \}$ and $\{2, 20, 10 \}$ nodes respectively for each layer making $c_x = 40$ and $c_u = 10$.
        The exploration noise used on the control is given by additive zero mean noise with a variance of $40 \%$ motor saturation decreasing at a rate of $0.9^{i+1}$.
        The decay weight on the information measure is given by $0.2^{i+1}$.
        The LQ weights are given by $\tilde{\mathbf{Q}} = \text{diag}([50.0, 1.0, 10.0, 0.1] + \vec{\bold{0}})$ where the first non-zero weights correspond to the states of the cart pendulum.
        A time horizon of $0.1 s$ is used with a sampling rate of $50$ Hz.
        The regularization weight $\tilde{\mathbf{R}} = 1\times 10^6$.

        For the 2-link robot example, we use a similar three layer network with a single hidden layer for $z_\theta$ and $v_\theta$ with $\{4, 20, 40 \}$ and $\{2, 20, 20 \}$ nodes respectively for each layer making $c_x = 40$ and $c_u = 10$.
        The exploration noise used on the control is given by additive zero mean noise with a variance of $40 \%$ motor saturation decreasing at a rate of $0.9^{i+1}$.
        The decay weight on the information measure is given by $0.2^{i+1}$.
        The LQ weights are given by $\tilde{\mathbf{Q}} = \text{diag}([10.0, 1.0, 20.0, 1.0] + \vec{0})$ where the first non-zero weights correspond to the states of the cart pendulum.
        A time horizon of $0.05 s$ is used with a sampling rate of $100$ Hz.
        The regularization weight $\tilde{\mathbf{R}} = \text{diag}( [1\times 10^6, 1\times 10^6])$.

    \subsection{SPRK Tracking in Sand}
    \label{app:SPRK}
        The SPRK robot is running a $30$ Hz sampling rate for control and state estimation.
        Control vectors are filtered using a low-pass filter to avoid noisy responses in the robot.
        The controller weights are defined as
        \[
        \tilde{\mathbf{Q}} = \text{diag}([60,60,5,5,\vec{\bold{1}} ]) \\
        \text{and } \mathbf{R} = \text{diag}([0.1, 0.1]).
        \]
        The control regularization is $\tilde{\mathbf{R}} = \mathbf{R}$.
        A weight of $80$ is added to the information measure.
        A time horizon of $0.5s$ is used to compute the controller.

        We run the active learning controller for $20$ seconds and then set the weight of the information measure to zero and track the end effector trajectory given by
        \[
        \begin{bmatrix}
        x(t)\\ y(t)
        \end{bmatrix} =
        \begin{bmatrix}
        0.5 \cos \left( t \right) + 1.12 \\
        0.3 \sin \left( 2 t\right) + 0.85
        \end{bmatrix}.
        \]

        In this example, the set of functions are chosen as a polynomial expansion of the velocity states $\bold{x} = \left[\dot{x}, \dot{y}\right]$ to the $3^{rd}$ order.
        The function observables are defined as
        \begin{equation*}
        z(x) = \left[ x, y, \dot{x}, \dot{y}, 1, \dot{x}^2, \dot{y}^2, \dot{x}^2\dot{y}, \ldots, \dot{x}^3\dot{y}^3 \right]^T \in \mathbb{R}^{18}
        \end{equation*}
        and
        \begin{equation*}
        v(x, u) = u \in \mathbb{R}^2.
        \end{equation*}

    \subsection{Sawyer Control}
    \label{app:sawyer}
        The Sawyer robot was run on a sampling rate of $100$ Hz.
        Control vectors are filtered using a low-pass filter to avoid noisy responses in the robot.
        The controller weights are defined as
        \[
        \tilde{\mathbf{Q}} = \text{diag}([200 \times \vec{\bold{1}}\in\mathbb{R}^{14}, \vec{\bold{1}} ]) \\
        \text{and } \mathbf{R} = \text{diag}([ 0.001 \times \vec{\bold{1}}\in\mathbb{R}^7]).
        \]
        The control regularization is $\tilde{\mathbf{R}} = \mathbf{R}$.
        A weight of $2000$ is added to the information measure.
        A time horizon of $0.5s$ is used to compute the controller.

        We run the active learning controller for $20$ seconds and then set the weight of the information measure to zero and track the end effector trajectory given by
        \[
        \begin{bmatrix}
        x(t)\\ y(t) \\z(t)
        \end{bmatrix} =
        \begin{bmatrix}
        0.8 \\
        0.1 \cos \left( 2 t \right) \\
        0.1 \sin \left( 4 t\right) + 0.4
        \end{bmatrix}.
        \]

        The functions of state using to compute the Koopman operator are defined as
        \begin{equation*}
        z(x) = \left[ \bold{x}^T, 1, \theta_1 \theta_2, \theta_2 \theta_3, \ldots, \theta_6^3 \theta_7^3,\dot{\theta}_1 \dot{\theta}_2, \ldots, \dot{\theta}_6^3 \dot{\theta}_7^3 \right]^T \in \mathbb{R}^{51}
        \end{equation*}
        with $v(u) = u \in \mathbb{R}^{7}$ as the torque input control of each individual joint and states $\bold{x}$ containing the joint angles and joint velocities.

\section{Proofs}

    \subsection{Proof of Proposition~\ref{prop:1}} \label{app:proof_prop1}
        \noindent \textbf{Proposition} ~\ref{thm:inf} :
        The sensitivity of switching from $\mu$ to $\mu_\star$ at any time $\tau \in \left[t_i, t_i+T \right]$
        for an infinitesimally small $\lambda$,
        (also known as the mode insertion gradient~\cite{egerstedt2003optimal, axelsson2008gradient}) is given by
        \begin{equation*}\label{eq:mode_insertion_gradient}
            \frac{\partial J}{\partial \lambda} \Big\vert_{\tau, \lambda=0} = \rho(\tau)^\top (f_2 - f_1)
        \end{equation*}
            where $z(t)$ is a solution to \ref{eq:koop_with_control} with $u(t) = \mu(z(t))$ and $z(t_i) = z(x(t_i))$,
            $f_2 = f(z(\tau), \mu_\star(\tau))$,
            $f_1 = f(z(\tau), \mu(z(\tau)))$, and
        \begin{equation*}
        \dot{\rho} =
        -\left(
            \frac{\partial \ell}{\partial z} + \frac{\partial \mu}{\partial z}^\top \frac{\partial \ell}{\partial u}
        \right)
        - \left(
            \frac{\partial f}{\partial z} + \frac{\partial f}{\partial u} \frac{\partial \mu}{\partial z}
        \right) ^\top \rho
        \end{equation*}
        subject to the terminal condition $\rho(t_i+T) = \frac{\partial}{\partial z} m(z(t_i+T))$.

        \begin{proof}
            Consider the objective (\ref{eq:gen_obj}) evaluated at a trajectory $z(t) \forall t \in \left[t_i, t_i+T \right]$
            generated from a dynamical system.
            Furthermore, assume that $z(t_i+T)$ is generated by a policy $\mu(z(t)) \forall t \notin \left[ \tau, \tau+\lambda \right]$
            and a controller $\mu_\star(t) \forall t \in \left[ \tau, \tau+\lambda \right]$ where $\tau$ is the time of application
            of control $\mu_\star$ and $\lambda$ is the duration of the control.
            Formally, $z(t_i + T)$ can be written as
            \begin{align}\label{eq:switched_koop}
                z(t_i + T) = z(t_i) &+ \int_{t_i}^{\tau} f(z(t), \mu(z(t))) dt \\
                                    &+ \int_{\tau}^{\tau + \lambda} f(z(t), \mu_\star(t) ) dt \nonumber\\
                                    &+ \int_{\tau + \lambda}^{t_i + T} f(z(t), \mu(z(t)) ) dt \nonumber,
            \end{align}
            where $f(z, u) : \mathbb{R}^{c_x} \times \mathbb{R}^{c_u} \to \mathbb{R}^{c_x}$ is a mapping which describes the time
            evolution of the state $z(t)$.

            Using (\ref{eq:switched_koop}) and (\ref{eq:gen_obj}), we compute the derivative of (\ref{eq:gen_obj}) with respect
            to the duration $\lambda$ of control $\mu_\star$ applied at any time $\tau \in \left[t_i, t_i+T\right]$:
            \begin{align}\label{eq:pre-lambda}
                \frac{\partial}{\partial \lambda} J \Bigg\vert_{\tau} & = \int_{\tau+\lambda}^{t_i + T}
                \left(
                    \frac{\partial \ell}{\partial z} + \frac{\partial \mu}{\partial z}^\top \frac{\partial \ell}{\partial u}
                \right)^\top \frac{\partial z}{\partial \lambda} dt.
            \end{align}
            where
            \begin{equation}\label{eq:recursive_dyn}
                \frac{\partial z(t)}{\partial \lambda} = f_2 - f_1
                + \int_{\tau + \lambda}^{t} \left( \frac{\partial f}{\partial z}
                + \frac{\partial f}{\partial u} \frac{\partial \mu}{\partial z} \right)^\top \frac{\partial z(s)}{\partial \lambda} ds
            \end{equation}
            such that $f_2 = f(z(\tau), \mu_\star(\tau))$, $f_1 = f(z(\tau), \mu(z(\tau)))$ are boundary terms from
            applying Leibniz's rule.

            Because (\ref{eq:recursive_dyn}) is a linear convolution with initial condition,
            $\frac{\partial z (\tau)} {\partial \lambda} = f_2 - f_1$,
            we are able to rewrite the solution to
            $\frac{\partial z(t)}{\partial \lambda}$ using a state-transition matrix $\Phi(t, \tau)$
            \cite{anderson2007optimal}
            with initial condition $f_2 - f_1$ as
            \begin{equation}
                \frac{\partial z(t)}{\partial \lambda} = \Phi(t, \tau) \left( f_2 - f_1 \right).
            \end{equation}
            Since the term $f_2 - f_1$ is evaluated at time $\tau$, we can write (\ref{eq:pre-lambda}) as
            \begin{align}\label{eq:no-limit-mode}
                \frac{\partial}{\partial \lambda} J \Bigg\vert_{\tau} = \int_{\tau+\lambda}^{t_i + T} \left(
                \frac{\partial \ell}{\partial z} + \frac{\partial \mu}{\partial z}^\top \frac{\partial \ell}{\partial u}
                \right)^\top \Phi(t, \tau) dt \left( f_2 - f_1 \right).
            \end{align}

            Taking the limit of (\ref{eq:no-limit-mode}) as $\lambda \to 0$ gives us the sensitivity of
            (\ref{eq:gen_obj}) with respect to switching at any time $\tau \in \left[t_i, t_i + T\right]$.
            We can further define the adjoint (or co-state) variable
            \[
                \rho(\tau)^\top =   \int_{\tau}^{t_i + T}
                \left(
                \frac{\partial \ell}{\partial x} + \frac{\partial \mu}{\partial x}^\top \frac{\partial \ell}{\partial u}
                \right)^\top \Phi(t, \tau) dt \in \mathbb{R}^{c_x}
            \]
            which allows us to define the mode insertion gradient~\cite{axelsson2008gradient} as
            \[
                \frac{\partial}{\partial \lambda} J \Big |_{t=\tau} = \rho(\tau)^\top \left( f_2 - f_1 \right)
            \]
            where
            \[
                \dot{\rho} =
                -\left(
                    \frac{\partial \ell}{\partial z} + \frac{\partial \mu}{\partial z}^\top \frac{\partial \ell}{\partial u}
                \right)
                - \left(
                \frac{\partial f}{\partial z} + \frac{\partial f}{\partial u} \frac{\partial \mu}{\partial z}
                \right) ^\top \rho
            \]
            subject to the terminal condition $\rho(t_i+T) = \frac{\partial}{\partial z} m(z(t_i+T))$.
        \end{proof}

    \subsection{Proof of Theorem~\ref{thm:inf}}\label{app:proof_thm1}
        \noindent \textbf{Theorem} ~\ref{thm:inf} :
          Given Assumption~\ref{ass:1} and dynamics~(\ref{eq:control_affine_koopman}), then the change in information
          \footnote{With respect to the information acquired from applying only $\mu(z)$.}
          $\Delta \mathbf{I}$ subject to (\ref{eq:sac_control}) is given to first order
          \begin{multline}
              \Delta \mathbf{I} \approx
              \left(
                  \Vert ( \kop_u v(x) \right)^\top \rho \Vert_{\tilde{\mathbf{R}}^{-1}}^2 + \ell_\text{task}(z, \mu_\star)
                  \\ - \ell_\text{task}(z, \mu)
              )  \mathfrak{I}_{\mu_\star}\mathfrak{I}_{\mu} + \mathcal{O}(\Delta t),
          \end{multline}
          where $\mathfrak{I}_{\mu_\star}$, $\mathfrak{I}_{\mu}$ is the T-optimality measure (\ref{eq:t-optimality}) from applying the control
          $\mu_\star$ and $\mu$.
        \begin{proof}
            First define (\ref{eq:gen_obj}) for a controller as
            \begin{equation}
                J(u(t)) = \int_{t_i}^{t_i+\Delta t} \frac{1}{\mathfrak{I}_u} + \ell_\text{task} (z(t), u(t)) dt
            \end{equation}
            where $\Delta t < T$ is a time duration, $z(t)$ is subject to the controller $u(t)$, and $\mathfrak{I}_u$ is the measure
            of information from applying the control $u$.
            If we consider the difference between $J(\mu_\star)$ and $J(\mu)$ where $\mu$ is a controller that minimizes
            $\ell_\text{task}(z,u)$, then
            \begin{align}
                J(\mu_\star) - J(\mu) {}& = \int_{t_i}^{t_i+\Delta t} \frac{1}{\mathfrak{I}_{\mu_\star}} -\frac{1}{\mathfrak{I}_{\mu}}
                + \ell_\text{task}(z, \mu_\star) - \ell_\text{task}(z, \mu) dt \nonumber \\
                \begin{split}
                {}& \approx \Delta t \left(  \frac{1}{\mathfrak{I}_{\mu_\star}} -\frac{1}{\mathfrak{I}_{\mu}} +
                \ell_\text{task}(z, \mu_\star) - \ell_\text{task}(z, \mu) \right) \\ & \hspace{5mm}+ \mathcal{O}(\Delta t).
                \end{split}
            \end{align}
            From Corollary~\ref{cor:neg_mode_insert} and that,
            \begin{equation*}
                \frac{\partial}{\partial \lambda} J \Delta t \approx J(\mu_\star) - J(\mu),
            \end{equation*}
             we can show that
            \begin{align}\label{eq:diff_j_mi}
                \frac{\partial}{\partial \lambda} J \Delta t & \approx J(\mu_\star) - J(\mu) \nonumber \\
                \begin{split}
                {}& \approx \Delta t \left(  \frac{1}{\mathfrak{I}_{\mu_\star}} -\frac{1}{\mathfrak{I}_{\mu}} +
                \ell_\text{task}(z, \mu_\star) - \ell_\text{task}(z, \mu) \right) \\ & \hspace{5mm} + \mathcal{O}(\Delta t).
                \end{split}
            \end{align}
            which we rearrange (\ref{eq:diff_j_mi}) and insert (\ref{eq:djdlam_neg}) to get
            \begin{align}\label{eq:simplified_diff_j}
                \begin{split}
                - \Vert \left(\kop_u v(x) \right)^\top \rho \Vert_{\tilde{\mathbf{R}}^{-1}}^2 & \approx
                 \left(  \frac{1}{\mathfrak{I}_{\mu_\star}} -\frac{1}{\mathfrak{I}_{\mu}} +
                \ell_\text{task}(z, \mu_\star) - \ell_\text{task}(z, \mu) \right) \\ & \hspace{5mm} + \mathcal{O}(\Delta t).
                \end{split} \nonumber \\
                \begin{split}
                & \approx \frac{\mathfrak{I}_\mu - \mathfrak{I}_{\mu_\star} + (\ell_\text{task}(z, \mu_\star) -
                 \ell_\text{task}(z, \mu))  \mathfrak{I}_{\mu_\star}\mathfrak{I}_{\mu}}{ \mathfrak{I}_{\mu_\star}\mathfrak{I}_{\mu}}
                  \\ & \hspace{5mm} + \mathcal{O}(\Delta t) .
              \end{split}
            \end{align}
            Setting $\Delta \mathbf{I} =  \mathfrak{I}_{\mu_\star} - \mathfrak{I}_{\mu}$ in (\ref{eq:simplified_diff_j}) and simplifying gives the relative information gain
            \begin{multline*}
                \Delta \mathbf{I} \approx  ( \Vert \left(\kop_u v(x)\right)^\top \rho \Vert_{\tilde{\mathbf{R}}^{-1}}^2 + \ell_\text{task}(z, \mu_\star) \\ - \ell_\text{task}(z, \mu) )  \mathfrak{I}_{\mu_\star}\mathfrak{I}_{\mu} + \mathcal{O}(\Delta t).
            \end{multline*}
        \end{proof}

\bibliographystyle{IEEEtran/IEEEtran}
\bibliography{bib}

\balance

\end{document}